\documentclass[10pt,a4paper]{article}
\usepackage[english]{babel}
\usepackage[textwidth=15cm]{geometry} % change textwidth
\usepackage{amsmath}
\usepackage{amsfonts}
\usepackage{nicefrac}       % compact symbols for 1/2, etc.
\usepackage{amssymb}        % mathematical symbols etc.
\usepackage{graphicx}		% display figures
\usepackage{enumitem} 		% get rid of space in itemize
\usepackage{epstopdf} 		% convert eps to PDF
\usepackage{stmaryrd}
\usepackage[title]{appendix} % appendix numbering
\usepackage[export]{adjustbox} % figures aligned at top
\graphicspath{{Figures/}}
\usepackage[round]{natbib}

\input{macrogilles_main}

\newcommand{\NE}{\texttt{NE}}
\newcommand{\STBnorm}{\texttt{STB-0}}
\newcommand{\STBweight}{\texttt{STB weight}}
\newcommand{\STBtheory}{\texttt{STB theory}}
\newcommand{\RKMSE}{\texttt{R-KMSE}}
\newcommand{\MTAconst}{\texttt{MTA const}}
\newcommand{\MTAstb}{\texttt{MTA stb}}
\newcommand{\PPJS}{\texttt{PP James-Stein}}
\newcommand{\MMD}{MMD}

\newcommand{\MSE}{\mathrm{MSE}}
\newcommand{\naive}{\wh{\mu}^{\mathrm{NE}}}
\newcommand{\intr}[1]{\llbracket #1 \rrbracket}
\newcommand{\taum}{\tau'}
\newcommand{\msenaive}{\ol{\sigma}}
\newcommand{\normop}[2][a]{\norm[#1]{#2}_{\mathrm{op}}}
\newcommand{\deff}{d_{\mathrm{eff}}}
\newcommand{\footremember}[2]{%
    \footnote{#2}
    \newcounter{#1}
    \setcounter{#1}{\value{footnote}}%
}
 
\begin{document}
\title{High-Dimensional Multi-Task Averaging and Application to Kernel Mean Embedding}
\author{Hannah Marienwald\footremember{uptu}{Universit\"at Potsdam, Potsdam, Germany, and Technische Universit\"at Berlin, Berlin, Germany.}
\and Jean-Baptiste Fermanian\footremember{ens}{\'Ecole Normale Sup\'erieure de Rennes, Rennes, France}
\and  Gilles Blanchard\footremember{ups}{Universit\'{e} Paris-Saclay,
%  Centre National de la Recherche Scientifique
%  (CNRS),
%  Institut National de Recherche en Informatique et en Automatique (Inria), and
  CNRS, Inria, Laboratoire de Math\'{e}matiques
d'Orsay, 91405, Orsay, France.}}
\date{}
\maketitle

\begin{abstract}
We propose an improved estimator for the multi-task averaging problem, whose goal is the joint estimation of the means of multiple distributions using separate, independent data sets. The naive approach is to take the empirical mean of each data set individually, whereas the proposed method exploits similarities between tasks, without any related information being known in advance. First, for each data set, similar or neighboring means are determined from the data by multiple testing.
Then each naive estimator is shrunk towards the local average of its neighbors.
We prove theoretically that this approach provides a reduction in mean squared error. This improvement can be significant when the dimension of the input space is large, demonstrating a ``blessing of dimensionality'' phenomenon. An application of this approach is the estimation of multiple kernel mean embeddings, which plays an important role in many modern applications. The theoretical results are verified on artificial and real world data.
\end{abstract}

\section{INTRODUCTION}
The estimation of means from i.i.d. data is arguably one of the
oldest and most classical problems in statistics. In this work we consider the problem of estimating
{\em multiple} means $\mu_1,\ldots,\mu_B$ of probability distributions $\mbp_1,\ldots,\mbp_B$,
over a common space $\cX=\mbr^d$ (or possibly a real Hilbert space $\cH$).
We assume that for each individual distribution $\mbp_i$, we observe an i.i.d. data set $X^{(i)}_{\bullet}$ of size $N_i$,
and that these data sets have been collected independently from each other.

In the rest of the paper,
we will call each such data set $X^{(i)}_{\bullet}$ a {\em bag}.
Mathematically, our model is thus 
\begin{equation}
  \label{eq:mainmodel}
  \begin{cases}
  X^{(i)}_{\bullet}:=(X_k^{(i)})_{1\leq k\leq N_i} \stackrel{i.i.d.}{\sim} \mbp_i, \; 1\leq i \leq B  ; \\
  (X^{(1)}_{\bullet},\ldots,X^{(B)}_{\bullet}) \text{ independent,}
\end{cases}
\end{equation}
where $\mbp_1,\ldots,\mbp_B$ are square integrable distributions on $\mbr^d$ which we call {\em tasks}, and our goal is the estimation
of their means
\begin{equation}
  \label{def:means}
  \mu_i := \ee{X\sim \mbp_b}{X} \in \mbr^d,\; 1\leq i \leq B.
\end{equation}
Given an estimate $\wh{\mu}_i$ of $\mu_i$, we will be interested in its
squared error $\norm{\wh{\mu}_i - \mu_i}^2$,
and aim at controlling it either with high
probability or in average (mean squared error, MSE):
\[\MSE(i,\wh{\mu}_i) := \e[1]{\norm{\wh{\mu}_i - \mu_i}^2};\]
this error can be considered either individually for each task  $\mbp_i$ or averaged over all tasks.

This problem is also known as multi-task averaging (MTA) \citep{feldman2014revisiting}, an instance of the multi-task learning (MTL) problem.
Prior work on MTL showed that learning multiple tasks jointly yields better performance compared to individual single task solutions \citep{caruana1997multitask, evgeniou2005learning, feldman2014revisiting}.

In this paper we adapt the idea of joint estimation to the multi-task averaging problem and will show that we can take advantage of some unknown {\em structure} in the set of tasks to improve the estimation. Here, by individual estimation we mean that our natural baseline is the naive estimator (NE) given by the simple empirical mean:
\begin{equation}
  \label{eq:naivedef}
  \naive_i := \frac{1}{N_i} \sum_{k=1}^{N_i} X_k^{(i)}; \;\;\;
  \MSE(i,\naive_i)
  = \frac{1}{N_i} \tr{\Sigma_i},
\end{equation}
where $\Sigma_i$ is the covariance matrix of $\mbp_i$.

Our motivation for considering this setting is the growing number of large databases taking the above form,
where independent bags corresponding to different but conceptually similar distributions are available;
for example, one can think of $i$ as an index for a large number of individuals, for each of which a number of observations (assumed to be sampled from an individual-specific distribution) are available,
say medical records, or online activity information collected by some governmental or corporate mass spying device.

While estimating means in such databases is of interest of its own, a particularly important motivation
to consider this setting is that of Kernel Mean Embedding (KME), a technique enjoying sustained attention in the statistical and machine learning
community since its introduction in the seminal paper of \citet{smola2007hilbert}; see \citet{Muandet2017overview} for an overview.
The KME methodology is used in a large number of applications, e.g. two sample testing \citep{gretton2012kernel}, goodness-of-fit \citep{chwialkowski2016kernel}, multiple instance or distributional learning for both supervised \citep{muandet2012learning, szabo2016learning} as well as unsupervised learning \citep{jegelka2009generalized}, to name just a few.

The core principle of KME is to represent the distribution $\mbp_Z$ of a random variable $Z$ via the mean of $X=\phi(Z)$, where
$\phi$ is a rich enough feature mapping from the input space $\cZ$ to a (reproducing kernel) Hilbert space $\cH$.
In practice, 
it is assumed that we have an i.i.d. bag $(Z_k)_{1 \leq k \leq N}$
from $\mbp$, which is used to estimate its KME.
Here we are interested again in the situation where a large number of independent data sets from different distributions are available, and we want to estimate their KMEs jointly.
This is, therefore, an instance of the model~\eqref{eq:mainmodel}, once we set $\cX :=\cH$ and $X_k^{(i)}:=\phi(Z_k^{(i)})$. 

\subsection{Relation to Previous Work}

The fact that the naive estimator~\eqref{eq:naivedef}
can be improved upon when multiple, real-valued means are to be
estimated simultaneously, has a long history in mathematical statistics.
More precisely, let us introduce the following isotropic Gaussian setting:
\begin{equation}
  \label{eq:Gauss}
\tag{GI} \mbp_i = \cN(\mu_i,I_i);\;  N_i = N, \qquad 1\leq i \leq B,
\end{equation}
on which we will come back in the sequel.

As shown in \citet{stein1956}, for $B=1$ with $d\geq 3$ the naive estimator is inadmissible, i.e.
there exists a strictly better estimator, with a lower MSE for any true
mean vector $\mu_1$. An explicit example of a better estimator is given by
the celebrated {\em James-Stein}  estimator (JSE) \citep{js1961}, which shrinks adaptively the
naive estimator towards $\mathbf{0}$, or more generally, towards an a priori fixed vector $\nu_0$.

The MTA problem was introduced by \citet{feldman2014revisiting}, who proposed an approach which regularizes the estimation such that similar tasks shall have similar means as well.
However, they assumed the pairwise task similarity to be given, which is unfeasible in most practical applications.
In addition to our own approach, we will also introduce a variation of theirs,
suitable for the KME framework, that \emph{estimates} the task similarity instead of assuming it to be known.
\citet{MarPon13} proposed a method based on spectral clustering of the tasks
and applying \citet{feldman2014revisiting}'s method separately on each cluster, but without theoretical analysis.

Variations of the JSE can be shown to yield possible improvements over the NE in
more general situations as well (see \citealp{fathi2020relaxing} for recent results
in non-Gaussian settings).
This has also been exploited for KME in \citet{muandet2016stein}, where a
Stein-type estimator in kernel space was shown to generally improve over naive KME estimation.
To the best of our knowledge, no shrinkage estimator for KME explicitly designed for or taking advantage of the MTA setting exists.

In the remainder of this work we will proceed as follows.
Section \ref{sec:method} introduces the basic idea of the approach and starts with a general discussion.
We will expose in Section \ref{sec:theory} a theoretical analysis proving
that the presented method improves upon the naive estimation in terms of squared error, possibly
by a large factor.
The general  theoretical results will be discussed explicitly for the Gaussian setting (Sec. \ref{sec:gaussiansetting}) and in the KME framework (Sec. \ref{sec:kmesetting}).
The approach is then tested for the KME setting on artificial and real world data in Section \ref{sec:experiments}. All proofs are found in the appendix Sections~\ref{apx:proofindep} to~\ref{sec:dev-kernel}, Appendix~\ref{apx:testedMethods} gives a detailed description of the estimators compared in the experiments, and Appendix~\ref{apx:gaussresults} presents additional numerical results in the Gaussian setting.

\section{METHOD}\label{sec:method}
The basic idea of our approach is to improve the estimation of a mean of a task by basing its estimation not on its own bag alone, but concatenating the samples from all bags it is \emph{sufficiently similar} to.
Since in most practical applications task similarity is not known, we will propose a statistical test that assesses task relatedness based on the given data.

\subsection{Overview of the Approach}
In the remainder of the paper we will use the notation $\intr{n}:=\set{1,\ldots,n}$.
For convenience of exposition, assume the~\eqref{eq:Gauss} setting.
In this case, the naive estimators
all have the same MSE, $\msenaive^2:=d/N$.
Fix a particular task (reindexed $i=0$) with mean $\mu_0$ that we wish to estimate, and assume for now
we are given the {\em side information} that for some constant $\tau>0$, it holds $\Delta_{0i}^2:= \norm{\mu_0-\mu_i}^2 \leq \tau \msenaive^2$ for some ``neighbor tasks'' $i \in \intr{V}$
(a subset of the larger set of $B$ tasks within range $\tau\msenaive^2$ to $\mu_0$, reindexed for convenience).
Consider the estimator $\wt{\mu}_0$ obtained by a simple average of 
neighbor naive estimators,
$\wt{\mu}_0 = \frac{1}{V+1} \sum_{i=0}^V \naive_i$.
 We can bound via usual bias-variance decomposition,
independence of the bags
and convexity of the squared norm:
\begin{equation}
  \MSE(0,\wt{\mu}_0) \;\; = \;\; \norm[3]{\frac{1}{V+1}\sum_{i=1}^V (\mu_0-\mu_i)}^2 + \frac{\msenaive^2}{V+1}
\;\; \leq \;\; {\msenaive^2}\frac{\paren{1 + V \tau}}{V+1}. \label{eq:rough}
\end{equation}
Thus, the above bound guarantees that $\wt{\mu}_0$ improves over $\naive_0$
whenever $\tau < 1$, and leads to a relative improvement
of order $\max(\tau,V^{-1})$.

In practice,  we {\em don't} have {\em any} a priori side information on the configuration
of the means. A simple idea is, therefore, to estimate the quantities $\Delta_{0i}^2$
from the data by an estimator $\wh{\Delta}_{0i}^2$  and select only those bags for which
$\wh{\Delta}_{0i}^2 \leq \wt{\tau} \msenaive^2$. 
This is in a nutshell the principle of our proposed method.

The deceptive simplicity of the above idea might be met with some deserved
skepticism.
One might expect that
the typical estimation error of $\wh{\Delta}^{2}_{0i}$ would be of the same order
as the MSE of the naive estimators.
Consequently, we could at best guarantee with high probability a bound of $\Delta_{0i}^2 \lesssim \msenaive^2$ for the estimated neighbor tasks, i.e. $\tau\approx 1$, which does not lead
to any substantial theoretical improvement when using~\eqref{eq:rough}.
The reason why the above criticism is pessimistic, even in the worst case, is the role
of the dimension $d$.
From high-dimensional statistics,  it is known
that the rate of {\em testing} for $\Delta^2_{0i}=0$, i.e. the minimum $\rho^2$ such
that a statistical test can detect $\Delta^2_{0i}\geq \rho^2$ with probability close to 1,
is faster than the rate of {\em estimation}, $\rho^2 \simeq \sqrt{d}/N = \msenaive^2/\sqrt{d}$
(see e.g. \citealp{baraud2002non,blanchard2018minimax}).
Thus, we can reliably determine
neighbor tasks with $\tau \approx 1/\sqrt{d}$.
Based on~\eqref{eq:rough}, we can hope again for an improvement of order up to $\mtc{O}(1/\sqrt{d})$ over NE,
which is significant even for a moderately large dimension.
In the rest of the paper,
we develop the idea sketched here more precisely and illustrate its consequences
on KME by numerical experiments.
The message we want to convey is that
the {\em curse} of higher dimensional data with its effect on MSE
can be to a limit mitigated by a {\em relative blessing} because we can take
advantage of neighboring tasks more efficiently.

\subsection{Proposed Approach}

Denote $\msenaive_i^2 = \MSE(i,\naive_i), i \in \intr{B}$.
Introduce the following notation: $\Delta_{ij}:=\norm{\mu_i-\mu_j}$.
In general, our approach assumes that we have at hand  a family of tests $(T_{ij})_{1 \leq i,j \leq B}$ for the null hypotheses
$H^0_{ij}: \Delta^2_{ij} > \tau \msenaive_i^2 $ against the alternatives $H^1_{ij}: \Delta^2_{ij} \leq \taum \msenaive_i^2$,
for $0\leq\taum<\tau$. The exact form of the tests will be discussed later for specific settings.

We denote the set of detected neighbors of task $i\in\intr{B}$ as $V_i := \set{j: T_{ij}=1, j \in \intr{B}}$;
we can safely assume $T_{ii}=1$ so that that $i \in V_i$ always holds and $\abs{V_i}\geq 1$.
We will also denote $V_i^*=V_i\setminus\set{i}$.
For $\gamma \in [0,1]$, define the modified estimator
\begin{equation}
  \label{eq:gammaest}
  \wt{\mu_i} := \gamma \naive_i + \frac{(1-\gamma)}{\abs{V_i}} \sum_{j \in V_i} \naive_j,
\end{equation}
which can be interpreted as a local shrinkage estimator pulling the naive estimator towards
the simple average of its neighbors.

\section{THEORETICAL RESULTS}\label{sec:theory}
We will assume
that the naive estimators defined by~\eqref{eq:naivedef} satisfy
\begin{equation}
  \label{eq:bdsigma}
  \max_{i \in \intr{B}} \MSE(i,\naive_i) \leq \msenaive^2.
\end{equation}
Define the notation
\begin{equation*}
  G(\tau):=  \set{(i,j) \in \intr{B}^2: \Delta_{ij}^2 \leq \tau \msenaive^2}; \;\;\;\;\;\;
             \ol{G}(\tau) :=  \set{(i,j) \in \intr{B}^2: \Delta_{ij}^2 \geq \tau \msenaive^2},
\end{equation*}
and two following events:
\begin{equation*}
  A(\tau) := \set[2]{ \max_{(i,j) \in \ol{G}(\tau)} T_{ij} =1 };
  \;\;\;\;\;\;
  B(\taum) := \set[2]{ \min_{(i,j) \in G(\taum)} T_{ij} =0 };
\end{equation*}
so $\prob{A(\tau)}$ is the collective false positive rate of the tests (or family-wise error rate)
while $\prob{B(\taum)}$ is the collective false negative rate to detect $\Delta^2_{ij} \leq \taum \msenaive^2$
(family-wise Type II error rate).

\subsection{A General Result under Independence of Estimators and Tests}
We start with a result assuming that the tests $(T_{ij})_{(i,j)\in \intr{B}^2}$ and the
estimators $(\naive_i)_{i \in \intr{B}}$ are independent. This can be achieved for instance
by splitting the original bags into two.

\begin{theorem}
  \label{th:indep}
  Assume model~\eqref{eq:mainmodel} holds as well as~\eqref{def:means}, and that~\eqref{eq:bdsigma} holds.
  Furthermore, assume that there exists a family of tests $(T_{ij})_{(i,j)\in \intr{B}^2}$ that is independent of   $(X^{(i)}_{\bullet})_{i\in \intr{B}}$.
  For a fixed constant $\tau>0$, consider the family of estimators $(\wt{\mu}_i)_{i \in \intr{B}}$ defined by \eqref{eq:gammaest} with
  respective parameters
  \begin{equation}
    \label{eq:optgamma}
    \gamma_i := \frac{\tau \abs{V_i^*}}{(1+\tau)\abs{V_i^*} + 1}.
  \end{equation}
  Then, conditionally to the event $A^c(\tau)$, it holds
  \begin{equation}
    \label{eq:res-indep-single}
    \forall i \in \intr{B}: \MSE(i,\wt{\mu}_i) \leq 
    \paren{\frac{\tau\abs{V_i^*}+1}{(1+\tau)\abs{V_i^*}+1}} \msenaive^2.
  \end{equation}
  Let $\cN$ denote the covering number of the set of means $\set{\mu_j, j \in \intr{B}}$ by balls of radius $\sqrt{\taum}\msenaive/2$.
  Then, conditionally to the events $A^c(\tau)$  and $B^c(\taum)$ (for $\taum<\tau$), it holds
  \begin{equation}
    \label{eq:res-indep-ave}
    \frac{1}{B} \sum_{i=1}^B \MSE(i,\wt{\mu}_b) \leq \paren{\frac{\tau}{\tau+1} + \frac{\cN}{B} \frac{1}{(\tau+1)}}\msenaive^2.
  \end{equation}
\end{theorem}
The proof can be found in the supplementary material. In a nutshell, conditional to the favorable event
$A^c(\tau)$, and because the tests are independent of the estimators, we can use the argument leading to~\eqref{eq:rough}, extended to take into account the shrinkage factor $\gamma$, and optimize the
value of $\gamma$ to obtain~\eqref{eq:optgamma}, \eqref{eq:res-indep-single}.
If $B^c(\taum)$ is satisfied as well, we can deduce~\eqref{eq:res-indep-ave} directly from~\eqref{eq:res-indep-single}.

\paragraph{Discussion.}
\begin{itemize}
\item The factor in the individual MSE bound~\eqref{eq:res-indep-single}
  is strictly less than $1$ as soon as $\abs{V_i}>1$. As the number of
  neighbors $\abs{V_i}$ grows, the factor is larger than but
  approaches $\tau/(1+\tau)$. Therefore, there is a general
  trade-off between $\tau$ and the number of neighbors in a
  neighborhood of radius $\sqrt{\tau}\msenaive$.
  Nevertheless, in order to aim at possibly significant
  improvement over naive estimation, a small value of $\tau$ should be taken.
\item The factor in the averaged MSE bound~\eqref{eq:res-indep-ave} is also always smaller than 1
  (as expected from the individual MSE bound). It has a nice interpretation in terms of the ratio
  $\cN/B$: if $\cN \ll B$, the improvement factor will be very close to $\tau/(1+\tau)$.
  Thus, we collectively can improve
  over the naive estimation wrt MSE as soon as the set of means has a small covering number (at scale $\sqrt{\taum}\msenaive/2$)
  in comparison to its cardinality. This condition can be met in different structural low complexity situations, e.g. clustered means, means being sparse vectors,  set of means on a low-dimensional manifold.
  Note that the method does not need information about said structure in advance and is in this sense
  adaptive to it.
\end{itemize}

\subsection{Using the Same Data for Tests and Estimation}

We now present a general result in the case where the estimators and tests are not assumed to be independent
(e.g. computed from the same data.) To this end we introduce the following additional events:
\begin{equation*}
  C(\tau):  \set[2]{\underset{i\neq j}{\max} | \langle \naive_i - \mu_i , \naive_j - \mu_j \rangle | > \tau \msenaive^2};
  \;\;\;\;\;\;
  C'(\tau):  \set{\underset{i}{\max} \| \naive_i - \mu_i \|^2 > \msenaive^2 + \tau \msenaive^2}. \\
\end{equation*}
\begin{theorem}
  \label{th:onesample}
  Assume that there exists a family of tests $(T_{ij})_{(i,j) \in \intr{B}^2}$.
  For a given $\tau>0$ consider the family of estimators $(\wt{\mu}_i)_{i \in \intr{B}}$ defined by \eqref{eq:gammaest} with
  respective parameters
  \begin{equation}
    \label{eq:optgamma-onesample}
    \gamma_i := \frac{\tau }{1+\tau}.
  \end{equation}
  Then, for $\taum\geq \tau$, with probability greater than $1- \prob{A(\tau)\cup B(\tau') \cup C(\tau) \cup C'(\tau)}$, it holds
   \begin{equation}
    \label{eq:res-onesample-single}
    \forall i \in \intr{B}:
    \norm{ \wt{\mu}_i -\mu_i}^2 \leq 2 \msenaive^2 \paren{ \tau + \frac{\tau + |V_i|^{-1} }{1+ \tau}}.
  \end{equation}
  Let $\cN$ denote the covering number of the set of means $\set{\mu_b, b \in \intr{B}}$ by balls of radius $\sqrt{\taum}\msenaive/2$. Then, with the same probability as above,
  it holds
  \begin{equation}
    \label{eq:res-onesample-ave}
    \frac{1}{B} \sum_{i=1}^B 
    \norm{ \wt{\mu}_i -\mu_i}^2 \leq 2\msenaive^2 \paren{ \tau + \frac{\tau}{1+ \tau } + \frac{\cN}{B} \frac{1}{1+\tau } }.
  \end{equation}
\end{theorem}
The interpretation of the above result is similar to that of Theorem~\ref{th:indep},
with the caveat that the factor in the MSE bound is not always bounded by 1 as earlier; but the
qualitative behaviour when $\tau$ is small, which is the relevant regime, is the same
as previously described.

\subsection{The {G}aussian Setting}\label{sec:gaussiansetting}

In view of the previous results, the crucial point is whether there exists a family of tests
such that the events $A(\tau), B(\taum), C(\tau), C'(\tau)$ have
small probability, for a value of $\tau$ significantly smaller than 1, and $\taum$ of the same order as $\tau$
(up to an absolute numerical constant).
This is what we establish now in the Gaussian setting.

\begin{proposition}
  \label{prop:testgauss}
  Assume~\eqref{eq:Gauss} is satisfied. For a fixed $\alpha \in (0,1)$, define the tests
  \begin{equation}
    \label{eq:deftestsgaus}
    T_{ij} = \ind{\norm{\naive_i -\naive_j}^2 \leq \zeta d/N},
  \end{equation}
  with $\zeta:=\paren{\sqrt{2+\tau} - 4 \sqrt{\delta}}^2$, where we put $\delta:=(2\log B + \log \alpha^{-1})/d$.

  Then, provided $\tau \geq \max(C\delta,\sqrt{C\delta})$ (with $C=10^3$), it holds $\prob{A(\tau)}\leq \alpha$,
$\prob{B(\taum)}\leq \alpha$ with $\taum=\tau/3$, $\prob{C(\tau)}\leq 2 \alpha$  and $\prob{C'(\tau)}\leq  \alpha$ .
\end{proposition}
The above result is significant in combination with Theorems~\ref{th:indep} and~\ref{th:onesample}
when $\delta$ is small,
which is the case if $\log(B)/d$ is small. The message is the following: in
a high-dimensional setting, provided $B \ll e^d$, we can reach a large improvement compared to the
naive estimators, if the set of means exhibits structure, as witnessed by a small covering number
at scale $d^{\frac{1}{4}}\sqrt{(\log B)/N}$. The best-case scenario is when all the means are tightly
clustered around a few values, so that $\cN$ is small but $B$ is large, then
the improvement in the MSE is by a factor of order $\sqrt{(\log B)/d}$.

\subsection{Methodology and Theory in the Kernel Mean Embedding Framework}\label{sec:kmesetting}

We recall that the principle of KME posits a reproducing kernel $k$ on an input space $\mtc{Z}$,
corresponding to a feature mapping $\Phi: \mtc{Z} \rightarrow \cH$, where $\cH$ is a Hilbert
space, with $k(z,z') = \inner{\phi(z),\phi(z')}$. The feature mapping $\phi$ can be extended to
{\em probability distributions} $\mbp$ on $\mtc{Z}$, via $\phi(\mbp) := \ee{Z\sim \mbp}{\phi(Z)}$, provided
this expectation exists, which can be guaranteed for instance if $\phi$ is bounded.
This gives rise to an extended kernel on probability distributions via $k(\mbp,\mbq) := \inner{\phi(\mbp),\phi(\mbq)}
=\ee{(Z,Z') \sim \mbp \otimes \mbq}{k(Z,Z')}$.

As explained in the introduction, if we have a large number of distributions $(\mbp_{i})_{i \in \intr{B}}$ for each of which
an independent bag $(Z_k^{(i)})_{ 1 \leq k \leq N_i}$ is available, and we wish to collectively estimate their KMEs,
this is an instance of the model~\eqref{eq:mainmodel}-\eqref{def:means} under the transformation $X_k^{(i)} := \phi(Z_k^{(i)})$.
The distributions $\mbp_i$ are replaced by their image distribution through $\phi$ s.t. $\mu_i = \phi(\mbp_i)$ and the naive estimators are $\naive_i = \phi(\wh{\mbp}_i)$, where
$\wh{\mbp}_i$ is the empirical measure associated to bag $Z_{\bullet}^{(i)}$. We will make the
assumption that the kernel is bounded,
$\sup_{z \in \mtc{Z}} k(z,z) = \sum_{z \in \mtc{Z}} \norm{\phi(z)}^2 \leq L^2$, resulting in the following
``bounded setting'':
\begin{equation}
  \label{eq:bounded}
  \tag{BS} \forall i \in \intr{B}: N_i =N \text{ and } \norm[1]{X^{(i)}_k}\leq L, \mbp_i-\text{a.s.}, k\in \intr{N}.
\end{equation}
(note in particular that we still assume that all bags have the same size for the theoretical results.)

As always for kernel-based methods, elements of the Hilbert space $\cH$ are an abstraction which are never explicitly
represented in practice; instead, norms and scalar products between elements, that can be written as linear combinations
of sample points, can be computed by straightforward formulas using the kernel. In this perspective, a
central object is the {\em inter-task Gram matrix} $K$ defined as
$K_{ij} := k(\mbp_i,\mbp_j)
= \inner{\mu_i,\mu_j}, (i,j) \in \intr{B}^2$.
In the framework of {\em inference on distributions}, the distributions $\mbp_{i}$ act as (latent) training points and
the matrix $K$ as the usual kernel Gram matrix for kernel inference.
In contrast to what is assumed in standard kernel inference,
 $K$ is not directly observed but approximated by $\wh{K}$ s.t. $\wh{K}_{ij} := \inner{\wh{\mu}_i,\wh{\mu}_j}$,
for some estimators $(\wh{\mu}_i)_{i\in \intr{B}}$ of the true KMEs. The following elementary proposition links the quality of approximation of the means with the corresponding inter-task Gram matrix:
\begin{proposition}
  \label{prop:gramfrineq}
  Assume the model~\eqref{eq:mainmodel}-\eqref{def:means} under the assumption $\norm[1]{X_k^{(i)}} \leq L$ for all $k,i$.
  Let $\wh{\mu}_i$ be estimators of $\mu_i$ bounded by $L$, and the matrices $K$ and $\wh{K}$ defined as the Gram matrices of $(\mu_i)_{i\in \intr{B}}$
  and $(\wh{\mu}_i)_{i \in \intr{B}}$, respectively. Then
    \begin{equation}
      \norm[2]{\frac{1}{B}(K-\wh{K})}^2_{\mathrm{Fr.}} \leq \frac{4 L^2}{B} \sum_{i \in \intr{B}} \norm{\mu_i -\wh{\mu}_i}^2,
    \end{equation}
    where $\norm{K}_{\mathrm{Fr.}} := \tr(KK^T)^{\frac{1}{2}}$ is the Frobenius norm.
\end{proposition}
This result further illustrates the interest of improving the task-averaged squared error.

In order to apply our general results Theorems~\ref{th:indep} and~\ref{th:onesample}, we must again find
suitable values of $\tau$ (as small as possible) and $\tau'$ (as close to $\tau$ as possible) so that
the probability of the events $A(\tau),B(\tau'),C(\tau), C'(\tau)$ is small,
in the setting~\eqref{eq:bounded}.
In that context, the role of the dimension $d$ will be played by the {\em effective dimension}
$\tr \Sigma / \norm{\Sigma}_{op}$, where $\Sigma$ is the covariance operator for
the variable $X$. More precisely, since this quantity can change from one source distribution to the the other,
we will make the following  assumption: there exists $\deff>0$ such that
\begin{equation}
  \label{eq:deffunif}
  \forall i\in \intr{B}:\qquad \deff \normop{\Sigma_i} \leq \tr{\Sigma_i}  \leq N \msenaive^2.
\end{equation}
Observe that in view of~\eqref{eq:naivedef}, the upper bound above is merely a reformulation of~\eqref{eq:bdsigma}
and, therefore, not a new assumption; the lower bound is.

We consider tests based on the unbiased estimate of the maximum mean discrepancy (MMD; note that
the MMD between tasks $i$ and $j$ is exactly $\Delta_{ij}^2$):
\begin{equation*}
  U_{ij} = \frac{1}{N(N-1)} \sum_{\substack{k,\ell =1\\k\neq \ell}}^{N} \paren{\inner[1]{X^{(i)}_k,X^{(i)}_\ell}
    + \inner[1]{X_k^{(j)},X^{(j)}_\ell}}
  - \frac{2}{N^2} \sum_{k,\ell=1}^{N} \inner[1]{X^{(i)}_k, X^{(j)}_\ell}.
\end{equation*}

\begin{proposition}
  \label{prop:testkernel}
  Consider model \eqref{eq:mainmodel}, the bounded setting~\eqref{eq:bounded} and
  assume~\eqref{eq:deffunif} holds. Define
  \begin{equation}
    \label{eq:defr}
    r(t) :=  5 \paren{\sqrt{\left(\frac{1}{\deff}+ \frac{L}{N\msenaive} \right)t} + \frac{Lt}{N\msenaive}},
\end{equation}
and
\begin{equation}
    \label{eq:deftau1}
    \tau_{\min}(t) := r(t) \max\paren{\sqrt{2},r(t) }.
\end{equation}

  For a fixed $t \geq 1$, define the tests $T_{ij}$ for $i,j$ in $\intr{B}^2$
  \begin{equation}
    \label{eq:simtest_kme}
    T_{ij} := \ind[1]{ U_{ij}< {\tau \msenaive^2}/{2} }.
  \end{equation}
Then, provided $  \tau  \geq 144 \tau_{min}(t)$ , it holds
\begin{equation*}
\prob{A(\tau) \cup B(\tau/4) \cup C(\tau/7) \cup C'(\tau/48) } \leq  14B^2e^{-t}\,.
\end{equation*}
\end{proposition}
The quantity $r(t)$ above (taking $t=\log(14B^2\alpha^{-1})$, where $1-\alpha$ is the target probability)
plays a role analogous to $\delta$ in the Gaussian setting (Proposition~\ref{prop:testgauss}).
As the bag size $N$ becomes sufficiently large, we expect $\msenaive=\cO(N^{-\frac{1}{2}})$
and, therefore, $\msenaive N = \cO(N^{\frac{1}{2}})$. Hence, provided $N$ is large enough,
the quantity $r(t)$ is mainly of the order $\sqrt{\log(B)/\deff}$. Like in the Gaussian case,
this factor determines the potential improvement with respect to the naive estimator, which can
be very significant if the effective data dimensionality $\deff$ is large.

From a technical point of view, capturing precisely the role of the effective dimension required us
to establish concentration inequalities for deviations of sums of bounded vector-valued
variables improving over the classical vectorial Bernstein's inequality of~\citet{PinSak86}.
We believe this result (see Corollary~\ref{cor:meanbounds} in the supplemental) to be of interest
of its own and to have potential other applications.

\section{EXPERIMENTS AND EVALUATION}\label{sec:experiments}
We validate our theoretical results in the KME setting\footnote{In the Gaussian setting, we report
numerical results in the Appendix~\ref{apx:gaussresults}.} on both synthetic as well as real world data.
The neighboring kernel means are determined from the tests as described in Eq.~\eqref{eq:simtest_kme}.
More specifically, in practice we use the modification that (i) we adapt
the formula for possibly unequal bag sizes, and (ii) in each test $T_{ij}$
we replace $\msenaive^2$ by the task-dependent
unbiased estimate
\begin{align}\label{eq:esttaskvar}
\wh{\MSE}(i,\naive_i) := \frac{1}{2N_i^2\left(N_i - 1\right)} \cdot \sum_{k \neq \ell}^{N_i} k(Z_k^{(i)}, Z_k^{(i)})
 - 2k(Z_k^{(i)}, Z_{\ell}^{(i)}) + k(Z_{\ell}^{(i)}, Z_{\ell}^{(i)}).
\end{align}
We analyze three different variations of our method which we call similarity test based (STB) approaches.
\STBnorm{} corresponds to Eq.~\eqref{eq:gammaest} with $\gamma = 0$.
\STBweight{} uses model optimization to find a suitable value for $\gamma$, whereas \STBtheory{} sets $\gamma$ as defined in Eq. \eqref{eq:optgamma}.
However, here we replaced $\tau$ with $c \cdot \tau$, where $c>0$ is a multiplicative constant, to allow for more flexibility.

We compare their performances to the naive estimation, \NE, and the regularized shrinkage estimator, \RKMSE, \citep{muandet2016stein} which also estimates the KME of each bag separately but shrinks it towards zero.
Furthermore, we modified the multi-task averaging approach presented in \citet{feldman2014revisiting} such that it can be used for the estimation of kernel mean embeddings.
Similar to our idea, this method shrinks the estimation towards related tasks.
However, they require the task similarity to be known.
Therefore, we test two options:
\MTAconst{} assumes constant similarity for each bag;
\MTAstb{} uses the proposed test from Eq. \eqref{eq:simtest_kme} to assess the bags for their similarity.
See Appendix~\ref{apx:testedMethods} for a detailed description of the tested methods.

In the presented results, each considered method has up to two tuning parameters that, in our experiments, are picked in order to optimize
averaged test error. Therefore, the reported results can be understood as close to ``oracle'' performance -- the best potential of each method when parameters are
close to optimal tuning. While this can be considered unrealistic for practice, a closely related situation can occur in the setting
where the user wishes to use the method on test bags of size $N$, and has at hand a limited number of training bags
of much larger size $N'\gg N$. From each such training bag, one can subsample $N$ points, use the method for estimation of the means of all bags of size $N$ (incl. subsampled bags), and monitor the error with respect to the means of the full training bags (of size $N'$, used as a ground truth proxy).
This allows a reasonable calibration of the tuning parameters.

\subsection{Synthetic Data}
The toy data consists of multiple, two-dimensional Gaussian distributed bags $Z^{(i)}_{\bullet}$ with fixed means but randomly rotated covariance matrices, i.e.
\begin{equation*}
  Z^{(i)}_{\bullet}
  \sim \mathcal{N}\left(\mathbf{0}, R(\theta_i) \Sigma
{R(\theta_i)}^T\right) = \mathbb{P}_i \;,
\;\;\;\;\;
\theta_i \sim \mathcal{U}(-\nicefrac{\pi}{4}, \nicefrac{\pi}{4}),
\end{equation*}
where the covariance matrix $\Sigma = \text{diag}(1,10)$ is rotated using rotation matrix  $R(\theta_i)$ according to angle $\theta_i$.
The different estimators are evaluated using the unbiased, squared \MMD{} between the
estimation $\wt{\mu}_i$ and $\mu_i$ as loss.
Since $\mu_i$ is unknown, it must be approximated by another (naive) estimation $\naive_i(Y^{(i)}_{\bullet})$ based on independent test bags $Y^{(i)}_{\bullet}$ from the same distribution as $Z^{(i)}_{\bullet}$,
with $|Y^{(i)}_{\bullet}| = 1000$.   The test bag $Y^{(i)}_{\bullet}$ has
much larger size than the training bag $Z^{(i)}_{\bullet}$, as a consequence
the estimator $\naive_i(Y^{(i)}_{\bullet})$ has a lower MSE than all considered estimators based on $Z^{(i)}_{\bullet}$, and can be used as a proxy for the true $\mu_i$.\footnote{
  Additionally, the estimation of the squared loss is unbiased if the diagonal entries of the Gram matrix will be included for $Z^{(i)}_{\bullet}$ but excluded for $Y^{(i)}_{\bullet}$.}
In order to guarantee comparability, all methods use a Gaussian RBF with the kernel width fixed to the average feature-wise standard deviation of the data.
Optimal values for the model parameter, e.g. $\zeta$ and $\gamma$ for \STBweight , are selected such that they minimize the estimation error averaged over 100 trials.
Once the values for the parameters are fixed, another 200 trials of data are generated to estimate the final generalization error.
Different experimental setups were tested:
\begin{itemize}[nosep]
\item[(a)] \textbf{Different Bag Sizes}
$B = 50$ and $N_i \in [10, 300]$  for all $i \in \intr{B}$,
\item[(b)] \textbf{Different Number of Bags}
$B \in [10, 300]$ and $N_i = 50$ for all $i \in \intr{B}$,
\item[(c)] \textbf{Imbalanced Bags} $B = 50$ and $N_1 = 10, \ldots, N_{50} = 300$,
\item[(d)] \textbf{Clustered Bags} $N_i, B = 50$  for all $i \in \intr{B}$ but
the Gaussian distributions are no longer centered around $\mathbf{0}$.
Instead, each ten bags form a cluster with the cluster centers equally spaced on a circle.
The radius of the circle is varied between 0 and 5, to model different degrees of overlap between clusters.
\end{itemize}

The results for the experiments on the synthetic data can be found in Figure \ref{fig:toy_results}(a) to (d).
The estimation of the KME becomes more accurate as the bag size per bag increases.
Nevertheless, all of the tested methods provide an increase in estimation performance over the naive estimation, although, the improvement for larger bag sizes decreases for \RKMSE{} and \MTAconst.
As expected, methods that use the local neighborhood of the KME yield lower estimation error when the number of available bags increases.
Interestingly, this decrease seems to converge towards a capping value,
which might reflect the intrinsic dimensionality of the data as indicated by Theorems~\ref{th:indep} and~\ref{th:onesample} combined with Proposition~\ref{prop:testkernel}.
Although we assumed equal bag sizes in the theoretical results, the proposed approaches provide accurate estimations also for the imbalanced setting.
Figure \ref{fig:toy_results}(c) shows that the improvement is most significant for bags with few samples, which is consistent with results on other multi-task learning problems (see e.g. \citealp{feldman2014revisiting}).
However, when the KME of a bag with many samples is shrunk towards a neighbor with few samples, the estimation can be deteriorated (compare results on (a) with those on (c) for large bag sizes).
A similar effect can be seen in the results on the clustered setting.
When the bags overlap, a bag from a different cluster might be considered as neighbor which leads to a stronger estimation bias.
When the tasks have similar centers or are strictly separated, the methods show similar performance to what is shown in Figure \ref{fig:toy_results}(b).

To summarize, \NE{} and \RKMSE{} give worst performances because they estimate the kernel means separately.
Even though \MTAconst{} assumes all tasks to be related, it improves the estimation performance even when the bags are not similar.
However, the methods that derive the task similarity from the local neighborhood achieve most accurate KME estimations in all of the tested scenarios, especially \STBweight{} and \STBtheory.

\begin{figure}[!ht]
\centering
\begin{tabular}[t]{ll}
 \multicolumn{1}{c}{(a) Different Bag Sizes} & \multicolumn{1}{c}{(b) Different Number of Bags}\\
\includegraphics[width=0.45\textwidth]{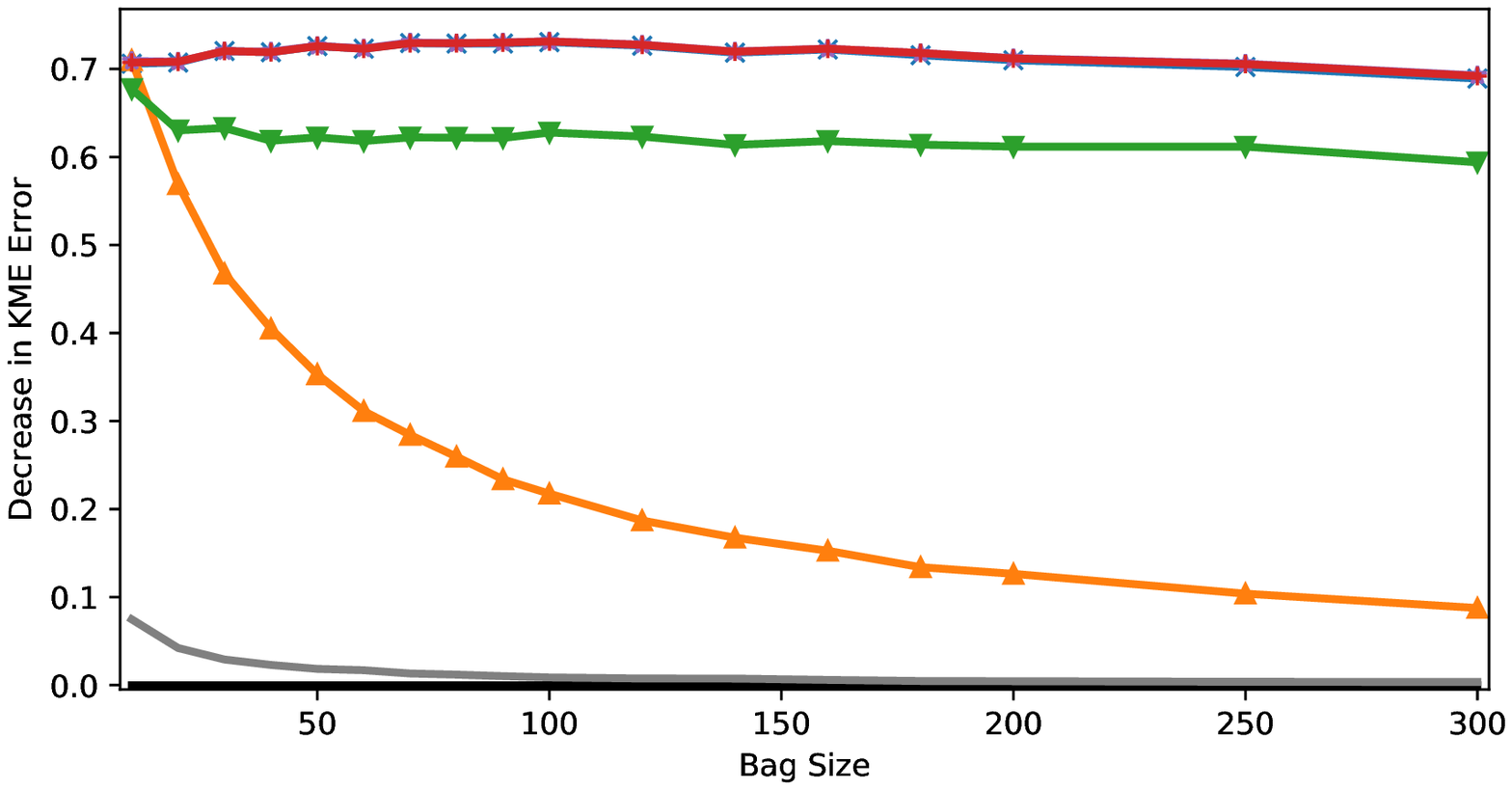} & \includegraphics[width=0.45\textwidth]{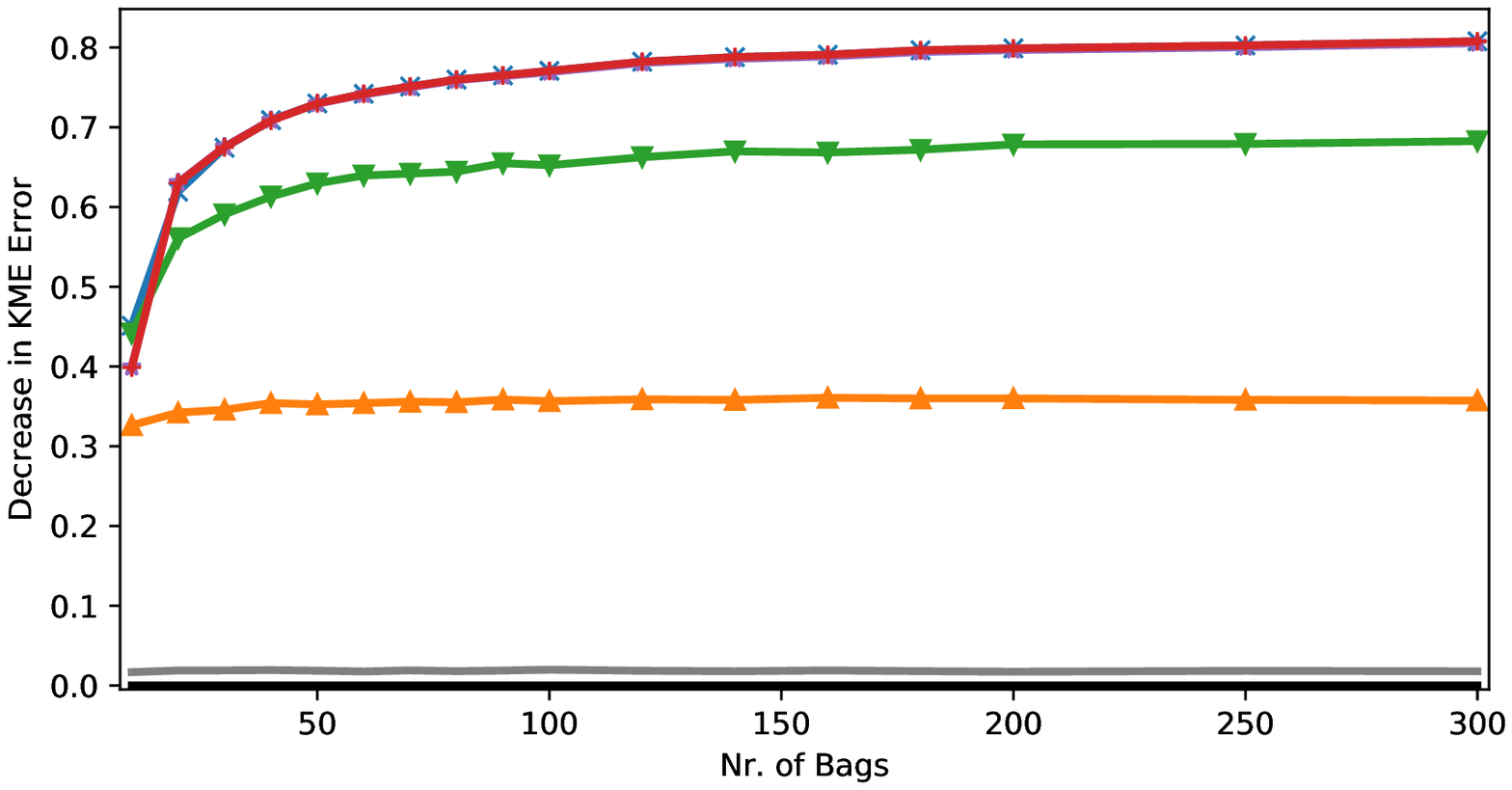}\\
[5pt]
\multicolumn{1}{c}{(c) Imbalanced Bags} & \multicolumn{1}{c}{(d) Clustered Bags}\\
\includegraphics[width=0.475\textwidth,valign=t]{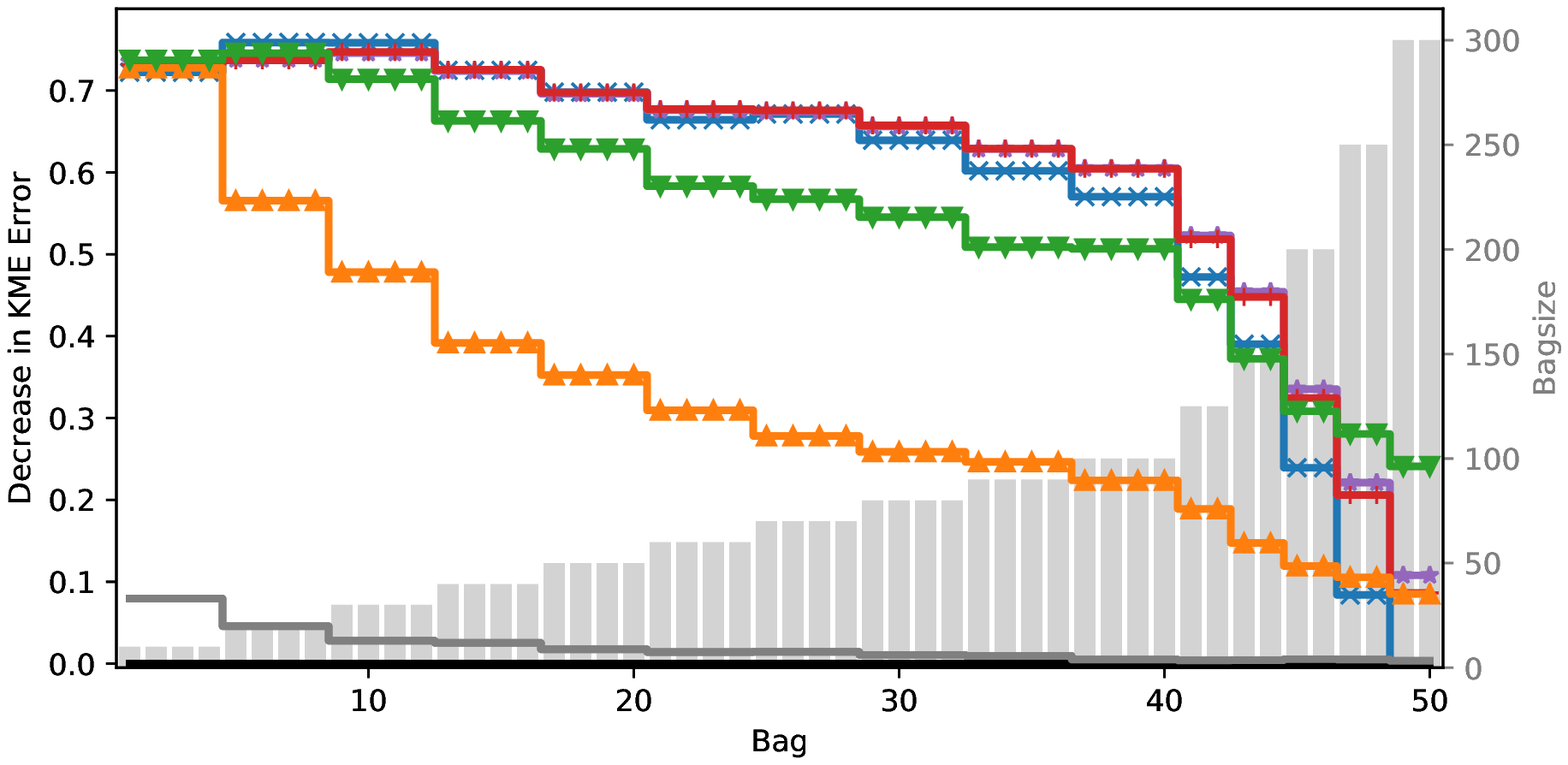}& \includegraphics[width=0.45\textwidth,valign=t]{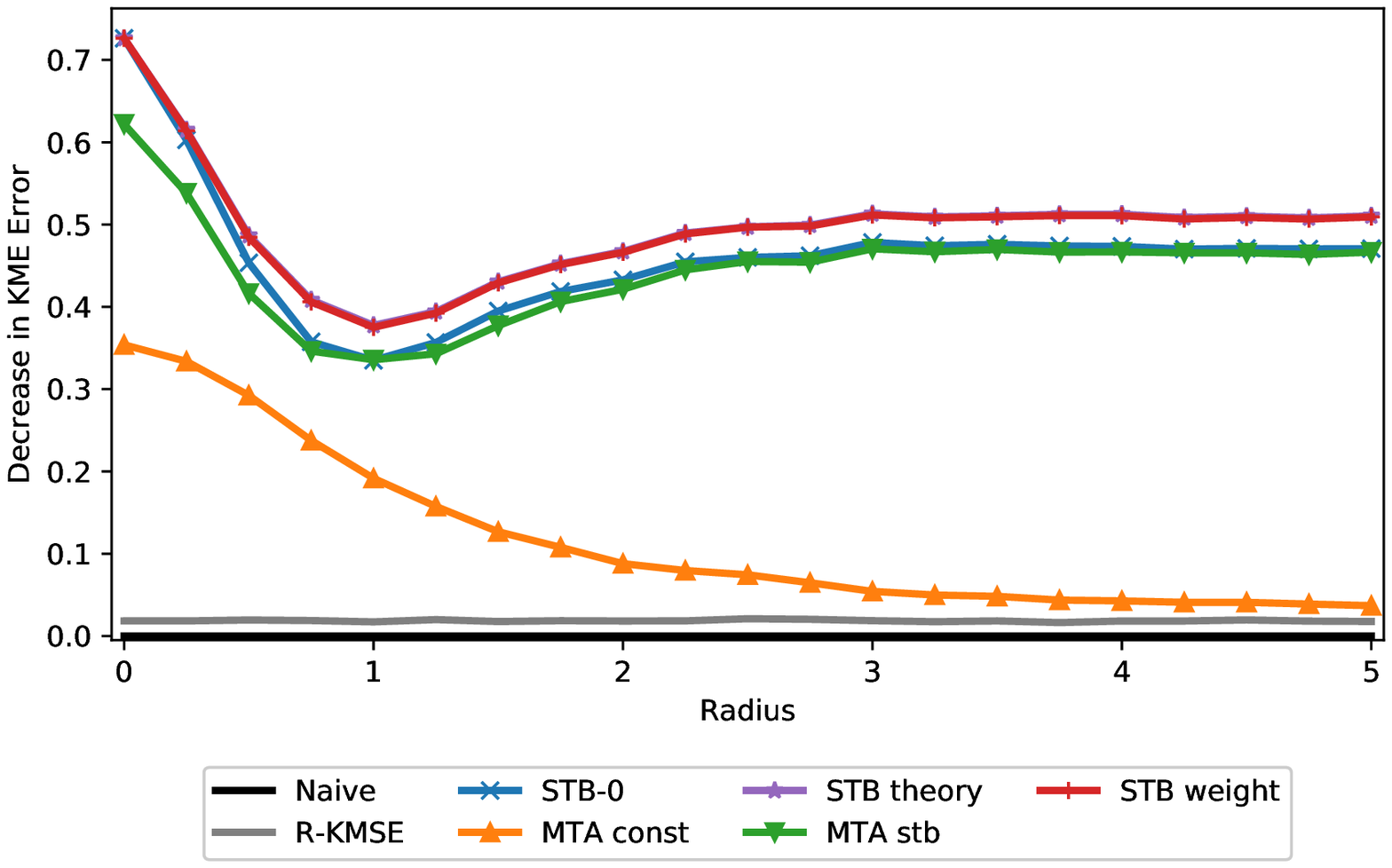}\\
\end{tabular}
\caption{
Decrease in KME estimation error compared to \NE{} in percent on experimental setups (a) to (d). Higher is better. \STBnorm, \STBweight{} and \STBtheory{} give similar results so that their results might be printed on top of each other.}
\label{fig:toy_results}
\end{figure}

\subsection{Real World Data}
We test our methods on a remote sensing data set.
The AOD-MISR1 data set is a collection of 800 bags with each 100 samples.
The samples correspond to randomly selected pixels from a MISR satellite, where each instance is formed by 12 reflectances from three MISR cameras.\footnote{We only use 12 out of 16 features because the remaining four are constant per bag.}
It can be used to predict the aerosol optical depth (AOD) which poses an important problem in climate research \citep{wang2011mixture}.

The data is standardized such that each of the features has unit standard deviation and is centered around zero.
In each out of the 100 trials, we randomly subsample 20 samples from each bag, on which the KME estimation is based.
This estimation is then compared to the naive estimation on the complete bag.
Cross-validation, with 400 bags for training and testing, is used to optimize for the model parameters of each approach and then estimate its error.
Again, all methods use a Gaussian RBF with the kernel width fixed to one.
The results are shown in Table \ref{tab:aod_results}.

\begin{table}[hb]
\caption{Decrease in KME estimation error compared to \NE{} in percent on the AOD-MISR1 data.} \label{tab:aod_results}
\begin{center}
\begin{tabular}{lrclrclr}
\textbf{METHOD}	&\textbf{\%} && \textbf{METHOD}	&\textbf{\%} && \textbf{METHOD}	&\textbf{\%}\\
\hline \\
\RKMSE	& $\phantom{0}8.83$ &$\phantom{000}$& \MTAconst & $13.92$ 	&$\phantom{000}$& \STBtheory & $21.83$\\
\STBnorm & $\phantom{0}1.43$ &$\phantom{000}$& \MTAstb	 & $17.17$ 	&$\phantom{000}$& \STBweight & $22.73$\\
\end{tabular}
\end{center}
\end{table}

Again, all of the methods provide a more accurate estimation of the KME than the naive approach.
The estimations given by \STBnorm{} are similar to those of \NE, because \STBnorm{} considers very few bags as neighbors.
This lets us conclude that the bags are rather isolated than overlapping.
\MTAstb, \STBweight{} and \STBtheory{} might give better estimations because they allow for more flexible shrinkage.
Again, \STBweight{} and \STBtheory{} are outperforming the remaining methods.

\section{CONCLUSION}
In this paper we proposed an improved estimator for the multi-task averaging problem.
The estimation is improved by shrinking the naive estimation towards the average of its neighboring means.
The neighbors of a task are found by multiple testing so that task similarities must not be known a priori.
Provided that appropriate tests exist, we proved that the introduced shrinkage approach yields a lower mean squared error for each task individually and also on average.
We show that there exists a family of statistical tests suitable for isotropic Gaussian distributed data or for means that lie in a reproducing kernel Hilbert space.
Theoretical analysis shows that this improvement can be especially significant
when the (effective) dimension of the data is large, using the property that the typical detection
radius of the tests is much better than the standard estimation error in high dimension.
This property is particularly important for the estimation of multiple kernel mean embeddings (KME) which is an interesting application relevant for the statistical and machine learning community.
The proposed estimator and the theoretical results can naturally be translated to the KME framework.

We tested different variations of the presented approach on synthetic and real world data and compared its performance to other state-of-the-art methods.
In all of the conducted experiments, the proposed shrinkage estimators yield the most accurate estimations.

Since the estimation of a KME is often only an intermediate step for solving a final task, as for example in distributional regression \citep{szabo2016learning}, further effort must be made to assess whether the improved estimation of the KME also leads to a better final prediction performance.
Furthermore, the results on the imbalanced toy data sets have shown that the shrinkage estimator particularly improves the estimation of small bags.
However, when the KME of a bag with many samples is shrunk towards a neighbor with low bag size, its estimation might be distorted.
Therefore, another direction for future work will be the development of a similarity test or a weighting scheme that take the bag size into account in a principled way. From a theoretical perspective, we also will investigate if the improvement factor with respect to the naive estimates is
optimal in a suitable minimax sense, and if the logarithmic factor $\log(B)$ and the number of tasks appearing in this factor can be removed
or alleviated in certain circumstances.

\subsubsection*{Acknowledgements}
The research of HM was funded by the German Ministry for Education and Research as BIFOLD (01IS18025A and 01IS18037A).
The research of GB has been partially funded by Deutsche Forschungsgemeinschaft (DFG) - SFB1294/1 - 318763901, and by the Agence Nationale de la Recherche (ANR, Chaire IA ``BiSCottE'').
GB acknowledges various inspiring and motivating discussions with A.~Carpentier, U.~Dogan, C.~Giraud, V.~Koltchinskii, G.~Lugosi, A.~Maurer, G.~Obozinski, C.~Scott.

\vfill
\pagebreak

\appendix
\section{Proof of Theorem~\ref{th:indep}}
\label{apx:proofindep}
  We argue conditional to the tests, below expectations are taken with respect to the
  samples $(X^{(b)}_{\bullet})_{b \in \intr{B}}$ only. Assume the event $A^c(\tau)$ holds,
  implying for all $i$:
  \begin{equation}
    \label{eq:neighborsA}
    j \in V_i \Rightarrow \Delta_{ij}^2 \leq \tau \sigma^2.
  \end{equation}
  Take $i=1$ without loss of generality, and denote $V=V_1, V^*=V_1\setminus\set{1},$ and $v=\abs{V_1}$.
  We also put $\eta=1-\gamma$. We use an argument similar to that leading to~\eqref{eq:rough}
  using independence of the bags, triangle inequality and~\eqref{eq:neighborsA}:
  \begin{align*}
  \MSE(1,\wt{\mu}_1)
  & = \e[4]{\norm[3]{(1-\eta)(\naive_1-\mu_1) + \frac{\eta}{v}\sum_{j \in V} (\naive_i-\mu_1)}^2}\\
  & = \frac{\eta^2}{v^2}\paren{ \norm[3]{\sum_{i \in V^*} (\mu_i-\mu_1)}^2 + \sum_{i\in V^*} \e{\norm{\mu_i-\naive_i}^2}}\\
  & \qquad  + (1-\eta(1-v^{-1}))^2\e{\norm{\naive_1-\mu_1}^2}\\
  & \leq \sigma^2\paren[2]{\frac{\eta^2}{v^2} \paren{(v-1)^2\tau + (v-1)}
   +  (1-\eta(1-v^{-1}))^2}\\
  & = \sigma^2\paren{\eta^2(1-v^{-1})\paren{(1-v^{-1})\tau +  1} - 2\eta(1-v^{-1}) +1 }.
\end{align*}
The optimal value of $\gamma=1-\eta$ is given by~\eqref{eq:optgamma} and gives rise to~\eqref{eq:res-indep-single}.

Assume additionally that $B^c(\taum)$ holds. Let $\eps:=\sqrt{\taum} \sigma/2$ and
let $\cC := \set{x_1,\ldots,x_\cN}$ be an $\eps$-covering of the set of means. Let $\pi(i)$
be the index of the element of $\cC$ closest to $\mu_i$, and $N_k := \set{b \in \intr{B}: \pi(b)=k}$,
$i \in \intr{\cN}$. By the triangle inequality, for any $i\in\intr{\cN}$, $b \in N_i$ one has $\abs{V_b} \geq \abs{N_{\pi(i)}}$. Hence averaging~\eqref{eq:res-indep-single} over $i$ we get
\begin{align*}
  \frac{1}{B} \sum_{b=1}^B \MSE(b,\wt{\mu}_b)
  &\leq \frac{\sigma^2}{B} \sum_{i \in \intr{B}} \frac{\tau(\abs[0]{N_{\pi(i)}}-1)+1}{
    1+(1+\tau)(\abs[0]{N_{\pi(i)}}-1)}\\
  & = \frac{\sigma^2}{B} \sum_{k \in \intr{\cN}} \frac{\abs[0]{N_k}(\tau(\abs[0]{N_k}-1)+1)}{
    1+(1+\tau)(\abs[0]{N_k}-1)}.
\end{align*}
The above take the form $\sum_k f(\abs{N_k})$, and it is straightforward to check that $f$ is convex.
Since it holds $1 \leq \abs{N_k} \leq B-\cN+1$ for all $k$, and $\sum_{k \in \intr{\cN}} \abs{N_k}=B$,
the maximum of the above expression is attained for an extremal point of this convex domain, i.e., by
symmetry, $N_1=B-\cN+1$ and $N_k=1$ for $k\geq 2$. Therefore
\begin{align*}
  \frac{1}{B} \sum_{b=1}^B \MSE(b,\wt{\mu}_b)
  &\leq \frac{\sigma^2}{B}\paren{(\cN-1) + \frac{(B-\cN+1)((B-\cN)\tau+1)}{(B-\cN)(1+\tau)+1}}\\
  & = \frac{\sigma^2}{B}\paren{\cN + \frac{(B-\cN)^2\tau}{(B-\cN)(1+\tau)+1 }}\\
  & \leq \sigma^2 \paren{\frac{\tau}{\tau+1} + \frac{\cN}{B} \frac{1}{\tau+1}}.
\end{align*}
\qed

\section{Proof of Theorem~\ref{th:onesample}}

We follow the same general line as in theorem \ref{th:indep}. Assume the event $A^c(\tau)\cap B^c(\tau') \cap C^c(\tau) \cap C'^c(\tau)$ holds. Take $i=1$ without loss of generality, and denote $V=V_1, V^*=V_1\setminus\set{1},$ and $v=\abs{V_1}$. We still put $\eta=1-\gamma$. Then
\begin{align*}
  \| \wt{\mu}_1 - \mu_1 \|^2 &= \norm[3]{(1-\eta)(\naive_1-\mu_1) + \frac{\eta}{v}\sum_{j \in V} (\naive_j-\mu_1)}^2 \\
  & \leq 2\left(\norm[3]{(1-\eta(1 - v^{-1}))(\naive_1-\mu_1) + \frac{\eta}{v}\sum_{j \in V^*} (\naive_j-\mu_j)}^2 + \frac{\eta^2}{v^2}\norm[3]{\sum_{j \in V^*} \mu_j - \mu_1}^2 \right).
\end{align*}
Let us upper bound the different terms. Because $j \in V$, we know that $\Delta_{j1} \leq \tau \msenaive^2$, so by the triangle inequality
\begin{equation*}
 \frac{\eta}{v}\norm[3]{\sum_{j \in V^*} \mu_j - \mu_1} \leq \frac{\eta}{v} \sum_{j \in V^*} \| \Delta_{ij} \| \leq \eta(1-v^{-1}) \sqrt{\tau} \msenaive.
\end{equation*}
Let us develop the other term :
\begin{multline*}
  \norm[3]{(1-\eta(1 - v^{-1}))(\naive_1-\mu_1) + \frac{\eta}{v}\sum_{j \in V^*} (\naive_j-\mu_j)}^2  \\
  = (1-\eta(1 - v^{-1}))^2\| \naive_1-\mu_1 \|^2 + \frac{2\eta(1-\eta(1 - v^{-1}))}{v} \sum_{j \in V^*} \langle \naive_1-\mu_1 , \naive_j - \mu_j\rangle \\
   + \frac{\eta^2}{v^2} \sum_{j \neq k \in V^*} \langle \naive_j-\mu_j , \naive_k - \mu_k\rangle + \frac{\eta^2}{v^2} \sum_{j \in V^*} \| \naive_j - \mu_j \|^2 \\
   \leq \msenaive^2 \Big[(1-\eta(1 - v^{-1}))^2(1+ \tau) + 2\eta(1-\eta(1 - v^{-1}))(1-v^{-1})\tau \\
   +\eta^2 (1-v^{-1})^2 \tau + \eta^2v^{-1}(1-v^{-1})(1+ \tau) \Big]\,.
\end{multline*}
Let us associate the two expressions, we obtain that :
\begin{equation*}
  \| \wt{\mu}_1 - \mu_1 \|^2  \leq 2\msenaive^2\Big[ \tau+1   -2(1- v^{-1}) \eta + (1- v^{-1})(1+  \tau ) \eta^2 \Big].
\end{equation*}
The expression is minimal when $\eta = (1+ \tau)^{-1}$. By the same arguments about using
covering numbers as in the proof of Theorem~\ref{th:indep}, we obtain that with probability greater than $1- \prob{A(\tau)\cup B(\tau') \cup C(\tau) \cup C'(\tau)}$ :
\begin{align*}
  \frac{1}{B}\sum_{i\in \intr{B}}  \| \wt{\mu_i} - \mu_i \|^2  & \leq \frac{2 \msenaive^2}{B} \sum_{i \in \intr{B}} \tau + \frac{\tau + |V_i|^{-1}}{1 + \tau} \\
  & \leq 2 \msenaive^2 \paren{ \tau + \frac{\tau}{1+ \tau } + \frac{\cN}{B} \frac{1}{1+\tau } }.
\end{align*}
\qed

\section{Proof of Proposition~\ref{prop:testgauss}}

Recall that we assume the~\eqref{eq:Gauss} model.
We first consider the behavior of a single test $T_{ij} = \ind{\norm{\naive_i -\naive_j}^2 \leq \zeta \msenaive^2}$, where $\msenaive^2:=d/N$, we also put $\Delta^2=\Delta^2_{ij}$ for short. The random variable $Z:= \naive_i -\naive_j$ is distributed as $\cN(\mu_i-\mu_j,2n^{-1} I_d)$ by independence of the bags. From classical concentration results for chi-squared variables
recalled as Proposition~\ref{prop:chisqdev} in Section~\ref{sec:dev}, for any $\alpha\in(0,1)$ either of the inequalities below hold
with probability $1-\alpha$:
\begin{equation}
  \label{eq:gausbnd}
  \sqrt{\Delta^2 + 2 \msenaive^2} - 4 \msenaive\sqrt{\frac{\log \alpha^{-1}}{d}} \leq \norm{Z} \leq
  \sqrt{\Delta^2 + 2 \msenaive^2} + 2 \msenaive\sqrt{\frac{\log \alpha^{-1}}{d}}.
\end{equation}
Put $\delta:=(\log \alpha^{-1})/d$ for short.

We start with analyzing Type I error: if $\Delta^2\geq \tau \msenaive^2$, then the above lower bound
implies $\norm{Z}^2 \geq \msenaive^2 \paren[1]{\sqrt{2+\tau} - 4 \sqrt{\delta}}^2$, so $T_{ij}=0$ if
we choose $\zeta:= \paren[1]{\sqrt{2+\tau} - 4 \sqrt{\delta}}^2$. By union bound over $(i,j) \in \intr{B}^2$,
with this choice we guarantee that $\prob{A(\tau)}\leq \alpha$ if we replace $\alpha$ by $\alpha B^2$
(i.e. take $\delta = (2 \log B + \log \alpha^{-1})/d$). This establishes the bound on family-wise type I error.

We now analyze type II error: assume now that we have picked $\zeta:= \paren[1]{\sqrt{2+\tau} - 4 \sqrt{\delta}}^2$, with
$\tau \geq \max(C \delta,\sqrt{C \delta})$,  $C=1000$, and assume $\Delta^2 \leq \taum \msenaive^2$. Then
assuming the upper bound in~\eqref{eq:gausbnd} is satisfied, we ensure $T_{ij}=1$ provided
\[
  \sqrt{\taum +2}  \leq \sqrt{\tau +2} -6 \sqrt{\delta}.
\]
Note that the condition on $\tau$ ensures that the above right-hand-side is positive. Taking squares
and further bounding, a sufficient condition for the above is $\taum \leq \tau - 12\sqrt{(2+\tau) \delta}$.
Using the condition on $\tau$, it holds
\[
  12\sqrt{(2+\tau) \delta} \leq 12 \sqrt{3C^{-1}}\tau \leq \frac{2}{3}\tau,
\]
hence $\taum \leq \tau/3$ is a sufficient condition. This ensures, by the union bound,
that $\prob{B(\tau')}\leq \alpha$ when replacing $\delta$ by $\delta' = (2 \log B + \log \alpha^{-1})/d$ as above.

We now turn to controlling the probability of the events $C(\tau)$ and $C'(\tau)$.
For fixed $i,j$ put $X_1 = \naive_i - \mu_i$, $X_2=\naive_j-\mu_j$. Under the~\eqref{eq:Gauss} model,
$X_1,X_2$ are independent $\cN(0,N^{-1}I_d)$. Applying the result of Proposition~\ref{prop:scalarproddev},
we obtain that for $\alpha \in (0,1)$, we have probability at least $1-2\alpha$:
\[
  \abs{\inner{X_i,X_j}} \leq \msenaive^2\paren[1]{\sqrt{2\delta}+\delta},
\]
where we have put $\delta:=(\log \alpha^{-1})/d$ as previously.
As soon as $\tau \geq \max(C \delta,\sqrt{C \delta})$, ($C\geq 1$) we obtain
$\abs{\inner{X_i,X_j}} \leq 3\tau \msenaive^2/\sqrt{C}$ on the above event,
implying that the event $C(\tau)$ is a fortiori satisfied for $C=10^3$.

From estimate~\eqref{eq:updev} in Proposition~\ref{prop:chisqdev}, we have with probability at
least $1-\alpha$:
\[
  \norm{X_1} \leq \msenaive^2\paren[1]{ 1 + \sqrt{2\delta}} \leq \msenaive^2\paren{ 1 + 2\tau C^{-1}},
\]
under the same condition on $\tau$ as above. As previously, by the union bound the above estimates are
true simultaneously for all $i,j$ with the indicated probabilities if
we replace $\delta$ by $\delta' = (2 \log B + \log \alpha^{-1})/d$, and $C'(\tau)$ is satisfied when
taking $C=10^3$.
\qed

\section{Results in the Bounded Setting (for KME Estimation)}

\subsection{Proof of Proposition~\ref{prop:gramfrineq}}
  \begin{align*}
  \norm[1]{(K-\wh{K})}_{\mathrm{Fr.}}^2
  & =  \sum_{(i,j) \in \intr{B}^2} ( \inner{\mu_i,\mu_j} - \inner{\wh{\mu}_i,\wh{\mu}_j})^2\\
  & =  \sum_{(i,j) \in \intr{B}^2} ( \inner{\mu_i-\wh{\mu}_i,\mu_j} + \inner{\wh{\mu}_i,\mu_j - \wh{\mu}_j})^2\\
  & \leq 2 \sum_{(i,j)\in \intr{B}^2} ( \inner{\mu_i-\wh{\mu}_i,\mu_j}^2 + \inner{\wh{\mu}_i,\mu_j - \wh{\mu}_j}^2)\\
  & \leq 2L^2 \sum_{(i,j)\in \intr{B}^2} ( \norm{\mu_i-\wh{\mu}_j}^2 + \norm{\mu_j - \wh{\mu}_j}^2)\\
  & \leq 4L^2 B \sum_{i \in \intr{B}} \norm{\mu_i-\wh{\mu}_j}^2.
\end{align*}\qed

\subsection{Proof of Proposition~\ref{prop:testkernel}}

Recall the notation
\begin{equation}
    r(t) =  5 \paren{\sqrt{\left(\frac{1}{\deff}+ \frac{L}{N\msenaive} \right)t} + \frac{Lt}{N\msenaive}},
\end{equation}
and
\begin{equation}
    \label{eq:defxi}
    \tau_{\min}(t) := r(t) \max\paren{\sqrt{2},r(t) }.
\end{equation}
  Introduce the notation $q(t):= \msenaive r(t)$; $\xi(t) := \msenaive^2 \tau_{\min}(t) = q(t)\max(\sqrt{2}\msenaive,q(t))$.
  Let $i,j \in \intr{B}^2$ be fixed and $t \geq 1$. We put $\tau = \lambda^2 \tau_{\min}(t)$ with $\lambda \geq 12$.
\smallbreak
Suppose that $\|\Delta_{ij}\|^2> \tau \msenaive^2 = \lambda^2 \tau_{\min} \msenaive^2= \lambda^2 \xi(t)$.
We use the concentration inequality~\eqref{eq:lowdev-Ustat} for bounded variables, proved in Section~\ref{sec:dev-kernel},
and obtain that with probability greater than $1- 8e^{-t}$,
and using the definition of $\xi(t)$:
  \begin{align*}
   U_{ij} \geq \|\Delta_{ij}\|^2 - 2 \|\Delta_{ij}\|q(t) - 8 \sqrt{2 \msenaive^2} q(t) -32 q^2(t) \geq \|\Delta_{ij}\| \Big( \|\Delta_{ij}\| - 2q(t) \Big) - 40 \xi(t)\,.
  \end{align*}
  (To be more precise,~\eqref{eq:lowdev-Ustat} proves the above estimate for the value of $q(t)$ defined by~\eqref{eq:defq-kernel},
  the value of $q(t)$ defined in the present proof is an upper bound for it, so the above also holds.)

   Observe $\| \Delta_{ij}\| \geq \lambda \sqrt{\xi(t)} \geq 12 \sqrt{\xi(t)} \geq 2 q(t)$.
   By monotonicity in $\| \Delta_{ij}\|$ under that condition, it holds $\|\Delta_{ij}\| \Big( \|\Delta_{ij}\| - 2q(t) \Big)\geq \sqrt{\lambda\xi(t)}(\lambda\sqrt{\xi(t)}-2q(t)) \geq \lambda(\lambda-2)\xi(t)$. That leads to
   \begin{align*}
    U_{ij} \geq ( \lambda^2 - 2\lambda -40) \xi(t) \geq (\lambda^2/2)\xi(t) = (\tau/2) \msenaive^2,
   \end{align*}
where  we have used that $\lambda^2 - 2\lambda -40 \geq \lambda^2/2$ for $\lambda\geq12$. So
  $$\prob{ \|\Delta_{ij}\|^2> \tau \msenaive^2   \quad \text{and} \quad  T_i =1} \leq 8e^{-t}.$$
  Suppose now $\|\Delta_{ij}\|^2< (\tau/4)\msenaive^2 = (\lambda^2/4) \xi(t) $. Then, according to the concentration ineqality
  \eqref{eq:updev-Ustat}, with probability greater than $1 - 8e^{-t}$, it holds
  \begin{align*}
     U_{ij}& \leq \|\Delta_{ij}\|^2 + 2 \|\Delta\|q(t) + 2 \sqrt{2 \msenaive^2} q(t) + 11q^2(t) \\
     & \leq \Big(\lambda^2/4 +\lambda+13 \Big) \xi(t) \\
     & \leq (\lambda^2/2) \xi(t) = (\tau/2) \msenaive^2.
  \end{align*}
    We have used that $\lambda^2/4 +\lambda+13 \leq \lambda^2/2$ for $\lambda\geq12$. So
  $$\prob{\|\Delta_{ij}\|^2 < \tau \msenaive^2/4   \quad  \text{and}  \quad  T_i =0 } \leq 2e^{-t}.$$
    An union bound over $(i,j) \in \intr{B}^2$ gives that
    \begin{equation*}
      \prob{ A(\tau) \cup B(\tau/4)} \leq 8B^2e^{-t}\,.
    \end{equation*}
  Remarking that
  \begin{gather*}
  \msenaive^2\tau/7 \geq 20 q(t) \max(q(t) , \sqrt{2\msenaive^2})\quad \text{and}\quad \msenaive^2\tau/48 \geq 3q(t) \max(q(t) , \sqrt{2\msenaive^2} )\geq 2q(t) \sqrt{2\msenaive^2} + q^2(t)
  \end{gather*}
  and using the concentration inequalities \eqref{eq:devVc} and \eqref{eq:devscprod-kernel} gives
  \begin{gather*}
    \prob{C(\tau/7)} \leq 6(B^2-B)e^{-t}\,, \quad \text{and} \quad \prob{C'(\tau/48)} \leq Be^{-t}\,.
  \end{gather*}
\qed

\section{Concentration Results in the Gaussian Setting}
\label{sec:dev}
\begin{proposition}
  \label{prop:chisqdev}
Let $Z$ be a normal $\cN(\mu,\sigma^2 I_d)$ random variable in $\mbr^d$. Then for any $t\geq 0$:
\begin{equation}
\label{eq:updev}
\prob{\norm{Z} \geq \sqrt{\norm{\mu}^2 + \sigma^2 d} + \sigma \sqrt{2t}} \leq e^{-t}\,,
\end{equation}
and
\begin{equation}
\label{eq:lowdev}
\prob{\norm{Z} \leq \sqrt{\norm{\mu}^2 + \sigma^2 d} - 2\sigma \sqrt{2t}} \leq e^{-t}.
\end{equation}
\end{proposition}
\begin{proof}
The stated inequalities are direct consequences of classical deviation inequalities for (noncentral) $\chi^2$ variables.
Put $\lambda:=\norm{\mu}^2$, then for the upper deviation bound, Lemma 8.1 of \citep{Bir01} states that
\[
\prob{\norm{Z}^2 \geq \lambda + d\sigma^2 + 2\sqrt{(2\lambda + d\sigma^2)\sigma^2t} +  2\sigma^2t} \leq e^{-t}\,,
\]
and we have
\[
\lambda + d\sigma^2 + 2\sqrt{(2\lambda + d\sigma^2)\sigma^2 t}
+  2\sigma^2 t \leq \paren{ \sqrt{\lambda + d \sigma^2}  + \sigma\sqrt{2 t}}^2\,,
\]
implying \eqref{eq:updev}. For the lower deviation bound, Lemma 8.1 of \citep{Bir01} states that
\[
\prob{\norm{Z}^2 \leq \lambda + d\sigma^2 - 2\sqrt{(2\lambda + d\sigma^2)\sigma^2t}} \leq e^{-t}\,,
\]
and we have
\[
\paren{\lambda + d\sigma^2 - 2\sqrt{(2\lambda + d\sigma^2)\sigma^2t}}_+
\geq  \sqrt{\lambda + d\sigma^2}\paren{\sqrt{\lambda + d\sigma^2} -2\sigma\sqrt{2t}}_+
\geq \paren{\sqrt{\lambda + d\sigma^2} -2\sigma\sqrt{2t}}_+^2,
\]
leading to \eqref{eq:lowdev}.
\end{proof}

\begin{proposition}
  \label{prop:scalarproddev}
Let $X_1,X_2$ be independent $\cN(0,\sigma^2 I_d)$ variables in dimension d. Then for any $t\geq 0$:
\begin{equation}
\label{eq:devscprod}
\prob{ \inner{X_1,X_2} \geq \sigma^2\paren{\sqrt{2dt} + t}} \leq e^{-t}.
\end{equation}
\end{proposition}
\begin{proof}
Without loss of generality assume $\sigma^2=1$. For two independent one-dimensional Gaussian variables $G_1,G_2$, one has
for any $\lambda \in [0,1]$:
\[
\e{\exp \lambda G_1 G_2} = \e{\e{\exp \lambda G_1 G_2|G_2}}= \e{\exp \frac{\lambda^2}{2} G_2^2} = \frac{1}{\sqrt{1-\lambda^2}}\,,
\]
so that
\[
\log \e{\exp \lambda \inner{X_1,X_2}} = \frac{d}{2} (-\log (1-\lambda^2)) \leq \frac{d}{2} \frac{\lambda^2}{(1-\lambda)}\,.
\]
Applying Lemma~8.2 of \citep{Bir01} gives~\eqref{eq:devscprod}.
\end{proof}

\section{Concentration Results in the Bounded Setting}\label{sec:dev-kernel}

Studying concentration in the kernel setting  means having concentration results of bounded variables
taking values in a separable Hilbert space. Recall that $k(x,y) = \inner{\phi(z),\phi(z')}_\cH$ for all $z$,$z'$ in $\cH$, so that if $k$ is bounded by $L^2$, then the map $\phi$ is bounded by $L$. To obtain concentration results, we will use Talagrand's inequality.

\begin{theorem}[Talagrand's inequality]
    \label{th:talagrand}
    Let $X_1^s,...,X_N^s$ be iid real random variables indexed by $s \in S$ where $S$ is a countable index set,
    and $L$ be a positive constant such that:
    \begin{gather*}
      \ee{}{X_k^s} = 0\,, \quad \text{and} \quad |X_k^s| \leq L \; a.s. \quad \forall k \in \intr{N},\, s\in S\,.
    \end{gather*}
    Let us note $Z  = \sup_{s \in S} \sum_{k=1}^{N} X_k^s$,
    then for all $t \geq 0$ :
    \begin{gather*}
      \prob{Z - \ee{}{Z} \geq 2 \sqrt{(2v + 16L\ee{}{Z})t} + 2Lt } \leq e^{-t} ;\\
      \prob{-Z + \ee{}{Z} \geq 2 \sqrt{(4v + 32L\ee{}{Z})t} + 4Lt } \leq e^{-t},
    \end{gather*}
    where $v = \sup_{s \in S} \sum_{k=1}^{N} \ee{}{(X_k^s)^2}$.
  \end{theorem}
  Talagrand's inequality appeared originally in~\cite{talagrand1996new}, with the above form (using additional
  symmetrization and contraction arguments from~\citealp{ledoux1991probability}) appearing in \cite{massart2000constants}. The constants in the upper deviation bound have been improved by
  \cite{rio2002inegalite} and~\cite{bousquet2002bennett}, however no such improvement is available
  for lower devations as far as we know. The above version is taken from~\cite{massart2007concentration}
  p. 169--170, (5.45) and (5.46) combined with (5.47) there.

Because a Hilbertian norm can be viewed as a supremum, we can use Talagrand's inequality to obtain a concentration inequality for the norm of the sum of bounded Hilbert-valued random variables.
\begin{proposition}
    \label{prop:chisqdev-kernel}
    Let $(Z_k)_{1 \leq k \leq N}$ be i.i.d. random variables taking values in a separable Hilbert space $\cH$,
    whose norm is bounded by $L$ a.s. Let $\mu$ and $\Sigma$ denote their common mean and covariance operator.
    Let
\begin{gather*}
  V = \left\| \frac{1}{N}\sum_{k=1}^{N} Z_k \right\|, \quad \text{and} \quad V_c = \left\| \frac{1}{N}\sum_{k=1}^{N} Z_k - \mu \right\|.
\end{gather*}
Then for any $t\geq 0$:
\begin{equation}
\label{eq:updev_kernel}
\prob{ V^2 \geq \| \mu\|^2 + \left( \ee{}{V_c} + q_\Sigma(t) \right)^2 + 2 \| \mu \| q_\Sigma(t)} \leq 2e^{-t}\,,
\end{equation}
and
\begin{equation}
\label{eq:lowdev_kernel}
\prob{ V^2 \leq \| \mu\|^2 + \left(\ee{}{V_c} - 2q_\Sigma(t) \right)_+^2 - 2 \| \mu \| q_\Sigma(t)} \leq 2e^{-t}\,,
\end{equation}
where
\begin{equation}
    \label{eq:defqsigma-kernel}
  q_\Sigma(t) = 2 \sqrt{\left(\frac{2\normop{\Sigma}}{N}+ 16L\frac{\sqrt{\tr \Sigma}}{N^{3/2}} \right)t} + \frac{2L}{N}t. \\
\end{equation}

\end{proposition}

\begin{proof}
  Let us denote $q(t)$ for $q_\Sigma(t)$ for this proof. We start with bounding the deviations of $V_c$. Observe that
  \begin{equation*}
    V_c = \underset{\|u\|_\cH =1}{ \sup} \frac{1}{N} \sum_{k=1}^{N} \inner{u,Z_k-\mu},
  \end{equation*}
  where the supremum can be restricted to $u$ in a dense countable subset $\cS$ of the unit sphere,
  since $\cH$ is separable. We can therefore apply Talagrand's inequality with $X_k^u:=N^{-1}\inner{u, Z_k-\mu}$;
  it holds $\abs{X_k^u} \leq L/N$, and note that since $\Sigma = \e{(Z-\mu)\otimes(Z-\mu)^*}$, it holds
  \[
    \e{(X_k^u)^2} = N^{-2}\e{\inner{u, Z_k-\mu}^2} = N^{-2} \inner{u,\Sigma u},
  \]
  so that $\sup_{u \in \cS} \sum_{k=1}^N \e{(X_k^u)^2} = N^{-1} \normop{\Sigma}$.
  Furthermore, $\e{V_c} \leq N^{-\frac{1}{2}} \sqrt{\tr \Sigma}$ by Jensen's inequality, which
  we use to further bound the deviation term by $q(t)$.

  By Theorem~\ref{th:talagrand}, with probability greater than $1-e^{-t}$ for $t \geq 0$, it holds
  \begin{equation}
    \label{eq:devVc}
    V_c \leq  \ee{}{V_c} +q(t)\,,
  \end{equation}
  and with probability greater than $1- e^{-t}$,
  \begin{equation}
    \label{eq:devlowVc}
    V_c \geq \ee{}{V_c} -2q(t)\,.
  \end{equation}
  We turn to bounding the deviations of $V^2-\norm{\mu}^2$. Observe
  \begin{equation}
    \label{eq:devlopv2}
    V^2 - \|\mu \|^2= V_c^2 + \frac{2}{N} \sum_{k=1}^{N} \inner{Z_k-\mu, \mu}.
  \end{equation}
  Using Bernstein's inequality for the variables $W_i=\inner{Z_i-\mu, \mu}$, satifying
  $\e{W_i}=0$, $\e{W_i^2} = \inner{\mu,\Sigma \mu}\leq \normop{\Sigma}\norm{\mu}^2$, and $\abs{W_i}\leq L\norm{\mu}$,
    we have that with probability greater than $1-e^{-t}$, for $t \geq 0$ :
    \begin{equation}
      \label{eq:bernsteinlinterm}
    \frac{1}{N} \sum_{i=1}^{N} \inner{Z_i-\mu, \mu} \leq \|\mu\| \left[ \sqrt{\frac{2\normop{\Sigma}t}{N}} + \frac{4Lt}{3N} \right] \leq \|\mu\|q(t)\,.
    \end{equation}
  Combining inequality~\eqref{eq:bernsteinlinterm}  with~\eqref{eq:devlopv2} and~\eqref{eq:devVc} gives that with probability greater than $1-2e^{-t}$ :
  \begin{align*}
    V^2 - \|\mu \|^2\leq   \left(\ee{}{V_c} +q(t)\right)^2 +  2\|\mu\|q(t)\,,
\end{align*}
and, combining~\eqref{eq:bernsteinlinterm},~\eqref{eq:devlopv2} and~\eqref{eq:devlowVc}, we have with probability greater than $1-2e^{-t}$ :
\begin{align*}
  V^2 - \|\mu \|^2\geq   \left(\ee{}{V_c} -2q(t)\right)_+^2 - 2\|\mu\|q(t)\,.
\end{align*}
\end{proof}

\begin{corollary}
    \label{cor:meanbounds}
    Using the setting and notation of Proposition~\ref{prop:chisqdev-kernel}, we have
    \begin{equation*}
      -2q_\Sigma(1) + \sqrt{\frac{\tr \Sigma}{N}} \leq \ee{}{V_c} \leq \sqrt{\frac{\tr \Sigma}{N}}\,.
    \end{equation*}
    As a consequence, for any $t >0$,
\begin{equation}
\label{eq:updev_kernel2}
\prob{ V^2 \geq \| \mu\|^2 + \paren[3]{ \sqrt{\frac{\tr \Sigma}{N}} + q_\Sigma(t) }^2 + 2 \| \mu \| q_\Sigma(t)} \leq 2e^{-t}\,,
\end{equation}
and for any $t \geq 1$,
\begin{equation}
\label{eq:lowdev_kernel2}
\prob{ V^2 \leq \| \mu\|^2 + \paren[3]{ \sqrt{\frac{\tr \Sigma}{N}}- 4q_\Sigma(t) }_+^2 - 2 \| \mu \| q_\Sigma(t)} \leq 2e^{-t}\,,
\end{equation}
\end{corollary}
{\bf Remark.} To the expert reader, we want to point out that the above concentration estimates are sharper than
the Bernstein's concentration inequality for vector random variables due to \citet{PinSak86} (Corollary~1 there) and which
has found many uses in the recent literature on kernel methods. The reason is that in \citeauthor{PinSak86}'s result,
which concerns deviations of the centered process $V_c$,
the deviation term (in factor of $t$) for $V_c$ is proportional to $\sqrt{\tr \Sigma/N}$.
The inequality of \citeauthor{PinSak86} also only bounds upper deviations.

In contrast, in the above result, the
term $\sqrt{\tr \Sigma/N} = \e[1]{\norm{V_c}^2}^{\frac{1}{2}}$ appears with constant $1$,
and the main deviation term (in factor of $t$) only involves $\sqrt{\normop{\Sigma}/N}$, which is better by a factor of $1/\sqrt{\deff}$.
We also obtain the informative lower deviation bound~\eqref{eq:lowdev_kernel2}.

To summarize, \citet{PinSak86}'s inequality controls the upper deviations of $V_c$ from zero in terms of a factor of its expectation,
while the above concentration inequalities
control the two-sided deviations of $V_c^2$ from its {\em expectation}, which is $\tr{\Sigma}/N$, in terms of a factor of its typical
deviation, which is $\normop{\Sigma}/N$.

This improvement makes the above bound first-order correct and mimic more closely the Gaussian chi-squared deviation
phenomenon of Proposition~\ref{prop:chisqdev}. This sharpness (and the fact that we get a control for two-sided deviations)
is crucial in order to be able to capture the behavior of the effective dimension,
see in particular Proposition~\ref{prop:ustatdev} below for the analysis of the MMD U-statistic, for which
the exact cancellation of the first order terms is paramount.

\begin{proof}
  The upper bound of the mean of $V_c$ is given directly by Jensen's inequality. For the lower bound, we can rewrite
  Talagrand's inequality \eqref{eq:devVc} equivalently under the following form: there exists
  $\xi$, an exponential random variable of parameter $1$, such that almost surely
  \begin{equation*}
    V_c \leq \ee{}{V_c} + q_\Sigma(\xi) = \ee{}{V_c} + \alpha \sqrt{\xi } + \beta \xi,
  \end{equation*}
  where $ \alpha$ and $\beta$ are given by \eqref{eq:defqsigma-kernel}. Taking the square and then the mean gives :
  \begin{align*}
    \ee{}{V_c^2} &\leq \ee{}{ \left( \ee{}{V_c} + \alpha \sqrt{\xi } + \beta \xi \right)^2} \\
    & \leq  \ee{}{  \left( \ee{}{V_c} + (\alpha+ \beta) \sqrt{\xi }  \right)^2 + 2 (\alpha+ \beta) \ee{}{V_c} \xi + (\alpha+ \beta)^2\xi^2 }\,.
  \end{align*}
  We can use now the concavity of the function $\xi \mapsto \left( \ee{}{V_c} + (\alpha+ \beta) \sqrt{\xi }  \right)^2$ and Jensen's inequality, obtaining
  \begin{align*}
    \ee{}{V_c^2} \leq  \left( \ee{}{V_c} + (\alpha+ \beta)\right)^2 + 2 (\alpha+ \beta) \ee{}{V_c} + 2(\alpha+ \beta)^2  \leq \left( \ee{}{V_c} + 2(\alpha+ \beta)\right)^2\,.
  \end{align*}
  Because $\ee{}{V_c^2} = \tr \Sigma/N$, and $(\alpha+\beta)=q_\Sigma(1)$ by definition, we obtain that
  \begin{equation*}
    \ee{}{V_c} \geq \sqrt{\frac{\tr \Sigma}{N}} -2q_\Sigma(1). 
  \end{equation*}
  If $t\geq 1$, it holds $q(t) \geq q(1)$ and we can plug in the above estimates for $\e{V_c}$ into~\eqref{eq:updev_kernel} and~\eqref{eq:lowdev_kernel}
  to obtain~\eqref{eq:updev_kernel2} and~\eqref{eq:lowdev_kernel2}, respectively (note that the condition $t\geq 1$ is only needed for the
  lower devation bound).
\end{proof}

\begin{proposition}
  \label{prop:scalarproddev-kernel}
Let $(X_k)_{1\leq k \leq N}\stackrel{i.i.d.}{\sim} X$ and $(Y_k)_{1 \leq k \leq N}\stackrel{i.i.d.}{\sim} Y$ be independent families of centered random variables bounded by $L$ in a separable Hilbert space $\cH$. Let $\Sigma_X$ and $\Sigma_Y$ be their respective covariance operators, $\msenaive^2$ and $\deff$ such that
\begin{gather*}
  \max(\tr \Sigma_X, \tr \Sigma_Y)/N \leq \msenaive^2\,; \\
  \min\left(\frac{\tr \Sigma_X}{ \|\Sigma_X\|_{op}}, \frac{ \tr \Sigma_Y}{ \|\Sigma_Y\|_{op}}\right) \geq \deff.
\end{gather*}
 Then for any $t\geq 0$:
\begin{equation}
\label{eq:devscprod-kernel}
\prob{ \inner{\frac{1}{N} \sum_{k=1}^N X_k ,\frac{1}{N} \sum_{k=1}^N Y_k} \geq 20q(t)\max(\msenaive , q(t))} \leq 6e^{-t},
\end{equation}
where
\begin{equation}
    \label{eq:defq-kernel}
  q(t) =  2 \sqrt{\left(\frac{4\msenaive^2}{\deff}+ 16L\frac{\sqrt{2\msenaive^2}}{N} \right)t} + \frac{2L}{N}t.
\end{equation}
\end{proposition}

\begin{proof}
  Let us remark that
  \begin{equation*}
    \inner{\frac{1}{N} \sum_{k=1}^N X_k ,\frac{1}{N} \sum_{k=1}^N Y_k} = \frac{1}{2N^2} \left[ \left\| \sum_{k=1}^{N}X_k + Y_k \right\|^2 -  \left\| \sum_{k=1}^{N}X_k \right\|^2 -  \left\| \sum_{k=1}^{N}Y_k \right\|^2 \right] \,.
  \end{equation*}
  So, by Corollary~\ref{cor:meanbounds}, with probability greater than $1-6e^{-t}$, for $t \geq 1$, and using $(a-b)^2_+ \geq a^2 -2ab$:
  \begin{align*}
    2\inner{\frac{1}{N} \sum_{k=1}^N X_k ,\frac{1}{N} \sum_{k=1}^N Y_k} &\leq  \left(\sqrt{\frac{\tr \Sigma_X+ \tr \Sigma_Y}{N}} +q(t) \right)^2 - \left(\sqrt{\frac{\tr \Sigma_X}{N}} -4q(t) \right)_+^2 \\
    & \;\;\;\;\;\; - \left(\sqrt{\frac{\tr \Sigma_Y}{N}} -4q(t) \right)_+^2 \\
    & \leq q(t)(19\msenaive + q(t)) \leq 20q(t)\max(\msenaive,q(t)).
  \end{align*}
\end{proof}

\begin{proposition}
    \label{prop:ustatdev}
Let $(X_k)_{1\leq i \leq N}\stackrel{i.i.d.}{\sim} X$ and $(Y_k)_{1 \leq i \leq N}\stackrel{i.i.d.}{\sim} Y$ be independent families of random variables bounded by $L$ in $\cH$. Let $\mu_x,\Sigma_X$ and $\mu_Y,\Sigma_Y$ denote their respective means and covariance operators. Let $U$ the statistic defined as
\begin{equation}
    \label{eq:Ustat}
  U = \frac{1}{N(N-1)} \sum_{\substack{k,\ell =1\\k\neq \ell}}^{N} \inner{X_k,X_\ell}_\cH - \frac{2}{N^2} \sum_{k,\ell=1}^{N} \inner{X_k, Y_\ell}_\cH +  \frac{1}{N(N-1)} \sum_{\substack{k,\ell =1\\k\neq \ell}}^{N} \inner{Y_k,Y_\ell}_\cH.
\end{equation}
Then for any $t\geq 1$, $N \geq 2$:
\begin{equation}\label{eq:updev-Ustat}
  \prob{ U \geq \| \mu_X - \mu_Y \|^2 + 2\| \mu_X - \mu_Y \|q(t)  + 2\sqrt{2\msenaive^2}q(t) + 11 q^2(t) } \leq 8e^{-t},
\end{equation}
and
\begin{equation}\label{eq:lowdev-Ustat}
  \prob{ U \leq \| \mu_X - \mu_Y \|^2 - 2\| \mu_X - \mu_Y \|q(t)  - 8\sqrt{2\msenaive^2}q(t) - 32q^2(t)} \leq 8e^{-t},
\end{equation}
where $q(t)$ is given by~\eqref{eq:defq-kernel}.
\end{proposition}

\begin{proof}
Observe that
 \begin{align*}
    U &= \left\| \frac{1}{N} \sum_{k=1}^N X_k - \frac{1}{N} \sum_{k=1}^N Y_k \right\|^2 \\
    & \;\;\; + \frac{1}{N-1}\left( \left\| \frac{1}{N} \sum_{k=1}^N X_k \right\|^2 +  \left\| \frac{1}{N} \sum_{k=1}^N Y_k \right\|^2 - \frac{1}{N} \sum_{k=1}^N \|X_k \|^2 - \frac{1}{N} \sum_{k=1}^N \| Y_k\|^2 \right)\\
     &=: \left\| \frac{1}{N} \sum_{k=1}^N X_k - \frac{1}{N} \sum_{k=1}^N Y_k \right\|^2 + \frac{1}{N-1}H \,.
 \end{align*}
 Using now the upper bound of Bernstein's inequality, since $\e[1]{\norm{X}^2} = \norm{\mu_X}^2 + \tr \Sigma_X$, with probability greater than $1-e^{-t}$
 it holds:
 \begin{equation*}
   \frac{1}{N} \sum_{k=1}^N \|X_k \|^2  \geq  \tr \Sigma_X + \|\mu_X\|^2  -\sqrt{2L^2 \msenaive^2t} - \frac{2L^2t}{3N} \,.
 \end{equation*}
 So using~\eqref{eq:updev_kernel2} (twice), with probability greater than $1- 6e^{-t}$ :
 \begin{align*}
 H \leq& \| \mu_X \|^2 + 2\| \mu_X \|q(t) + \paren[3]{\sqrt{ \frac{\tr \Sigma_X}{N}} + q(t) }^2 +  \| \mu_Y \|^2 + 2\| \mu_Y \|q(t) + \paren[3]{ \sqrt{ \frac{\tr \Sigma_Y}{N}} + q(t) }^2 \\
 & \qquad - \tr \Sigma_X - \|\mu_X\|^2  +\sqrt{2L^2 \msenaive^2t} + \frac{2L^2t}{3N}  - \tr \Sigma_Y - \|\mu_Y\|^2  +\sqrt{2L^2 \msenaive^2t} + \frac{2L^2t}{3N} \\
  \leq& - (N-1)/N \Big( \tr \Sigma_X + \tr \Sigma_Y\Big) + 4Lq(t) +   4\sqrt{\msenaive^2}q(t) + 2 q^2(t) +2\sqrt{2L^2 \msenaive^2t} + \frac{4L^2t}{3N} \\
  \leq &  - (N-1)/N \Big( \tr \Sigma_X + \tr \Sigma_Y\Big) + (2+4N)q^2(t)\,.
 \end{align*}
Using again~\eqref{eq:updev_kernel2}, and $N\geq 2$, with probability greater than  $1- 8e^{-t}$ :
 \begin{align*}
   U &\leq \| \mu_X - \mu_Y \|^2 + 2\| \mu_X - \mu_Y \|q(t) + \paren[3]{ \sqrt{ \frac{\tr \Sigma_X + \tr \Sigma_Y}{N}} + q(t) }^2 -  \frac{ \tr \Sigma_X + \tr \Sigma_Y}{N} + 10q^2(t)  \\
    & \leq \| \mu_X - \mu_Y \|^2 + 2\| \mu_X - \mu_Y \|q(t)  + 2\sqrt{2\msenaive^2}q(t) + 11q^2(t)\,,
 \end{align*}
 which is~\eqref{eq:updev-Ustat}.

 We proceed similarly for lower deviations of $U$: using again Bernstein's inequality and~\eqref{eq:lowdev_kernel2}, with probability greater than $1- 6e^{-t}$, and using $(a-b)_+^2 \geq a^2 - 2ab$:
 \begin{align*}
 H & \geq \| \mu_X \|^2 - 2\| \mu_X \|q(t) + \paren[3]{ \sqrt{ \frac{\tr \Sigma_X}{N}} - 4q(t) }_+^2 +  \| \mu_Y \|^2 - 2\| \mu_Y \|q(t) + \paren[3]{ \sqrt{ \frac{\tr \Sigma_Y}{N}} - 4q(t) }_+^2 \\
 & \qquad -\tr \Sigma_X - \|\mu_X\|^2  -\sqrt{2L^2 \msenaive^2t} - \frac{2L^2t}{3N} - \tr \Sigma_Y - \|\mu_Y\|^2  -\sqrt{2L^2 \msenaive^2t} - \frac{2L^2t}{3N} \\
  & \geq - (N-1)/N \paren{ \tr \Sigma_X + \tr \Sigma_Y} -16Nq^2(t)\,,
 \end{align*}
 which implies, using again~\eqref{eq:lowdev_kernel2}, and $N\geq 2$, that with probability greater than  $1- 8e^{-t}$ it holds:
  \begin{align*}
   U &\geq \| \mu_X - \mu_Y \|^2 - 2\| \mu_X - \mu_Y \|q(t) + \paren[3]{ \sqrt{ \frac{\tr \Sigma_X + \tr \Sigma_Y}{N}} - 4q(t) }_+^2 -  \frac{ \tr \Sigma_X + \tr \Sigma_Y}{N} - 16q^2(t) \\
     &\geq \| \mu_X - \mu_Y \|^2 - 2\| \mu_X - \mu_Y \|q(t)  - 8\sqrt{2\msenaive^2}q(t) - 32q^2(t)\,,
  \end{align*}
  which is~\eqref{eq:lowdev-Ustat}.

\end{proof}

\section{Details on the Tested Methods in the Numerical Experiments}\label{apx:testedMethods}
In the following, the methods that are tested in the experiments are described in more detail.
Recall, that $V_i := \set{j: T_{ij}=1, j \in \intr{B}}$ and let $T_{ij}$ be defined as in Eq \eqref{eq:simtest_kme}, i.e. $V_i$ holds the neighboring kernel means of bag $i$.
All of the methods give KME estimations of the form
\begin{equation*}
\tilde{\mu}_i := \sum_{j \in \intr{B}} \omega_{ij} \cdot \naive_j,
\end{equation*}
where the definition of the weighting $w_{ij}$ depends on the applied method.
\begin{enumerate}
\item \NE{} considers each bag individually. Therefore, the weighting is simply
\begin{equation*}
\omega_{ij} = \begin{cases}
1, &\text{for   } i = j \\
0, &\text{otherwise.}
\end{cases}
\end{equation*}

\item \RKMSE{} was proposed by \citet{muandet2016stein}.
It estimates each KME individually but shrinks it towards 0.
The amount of shrinkage depends on the data and is defined as
\begin{equation*}
\omega_{ij} = \begin{cases}
1-\frac{\lambda}{1+\lambda}, &\text{for   } i = j \\
0, &\text{otherwise}
\end{cases}
\end{equation*}
where
\begin{equation*}
\lambda = \frac{\varrho - \rho}{(\nicefrac{1}{N_b} - 1)\varrho + (N_b - 1)\rho}
\end{equation*}
with $\varrho = \nicefrac{1}{N_i} \sum_{k=1}^{N_i} k(Z_k^{(i)},Z_k^{(i)})$ and
$\rho = \nicefrac{1}{N_i^2} \sum_{k,\ell=1}^{N_i} k(Z_k^{(i)},Z_{\ell}^{(i)})$.

\item \STBnorm{} is described in Eq. \eqref{eq:gammaest} with $\gamma$ set to $0$, i.e.
\begin{equation*}
\omega_{ij} = \begin{cases}
\frac{1}{\vert V_i \vert}, &\text{for   } j \in V_i \\
0, &\text{otherwise.}
\end{cases}
\end{equation*}

\item \STBtheory{} is defined by Eq. \eqref{eq:gammaest}.
It uses the optimal value for $\gamma$ as described in Eq. \eqref{eq:optgamma} that was proven to be optimal.
Here, $\tau$ is replaced by its empirical counterpart $\zeta$ and another multiplicative constant $c>0$ was added to allow for more flexibility.
Its specific value must be found using model optimization.
\begin{equation*}
\omega_{ij} = \begin{cases}
\gamma + \frac{1-\gamma_i}{\vert V_i \vert}, &\text{for   } i = j \\
\frac{1-\gamma_i}{\vert V_i \vert}, &\text{for   } i \neq j,  j \in V_i \\
0, &\text{otherwise}
\end{cases}
\end{equation*}
with
\begin{equation*}
\gamma_i = \frac{c \cdot \zeta \cdot (\vert V_i \vert - 1)}{(1 + c \cdot \zeta) \cdot (\vert V_i \vert - 1) + 1}.
\end{equation*}

\item \STBweight{} is also described by Eq. \eqref{eq:gammaest} but the optimal value of $\gamma$ is found by model optimization
\begin{equation*}
\omega_{ij} = \begin{cases}
\gamma + \frac{1-\gamma}{\vert V_i \vert}, &\text{for   } i = j \\
\frac{1-\gamma}{\vert V_i \vert}, &\text{for   } i \neq j,  j \in V_i \\
0, &\text{otherwise.}
\end{cases}
\end{equation*}

\item \MTAconst{} is based on a multi-task averaging approach described in \citet{feldman2014revisiting} which we translated to the KME framework as
\begin{equation}\label{eq:mtaweight}
\omega_{ij} = \left({\left(I + \frac{\gamma}{B} D \cdot L(A)\right)}^{-1}\right)_{ij}.
\end{equation}
Here, $D = \text{diag}\left( (E_i)_{i \in \intr{B}}\right)$ as defined in Eq. \eqref{eq:esttaskvar} and $L(A)$ denotes the graph Laplacian of task-similarity matrix $A$.
For \MTAconst{} the similarity is assumed to be constant, i.e.
$A = a \cdot (\mathbf{1} \mathbf{1}^T)$
with $a = \frac{1}{B (B-1)} \sum_{i,j \in \intr{B}} \norm{\naive_i -\naive_j}^2_\mathcal{H}$.
Again, the optimal value for $\gamma$ must be found using model optimization.

\item \MTAstb{} is defined as in Eq. \eqref{eq:mtaweight}.
In contrast to \MTAconst, the similarity matrix $A$ is defined as
\begin{equation*}
A_{ij} = \begin{cases}
1, &\text{for   } j \in V_i \\
0, &\text{otherwise.}
\end{cases}
\end{equation*}
\end{enumerate}
The methods \STBnorm, \STBweight, \STBtheory{} and \MTAstb{} use all the similarity test defined by $T_{ij}$ which depends on $\zeta$.
Nevertheless, the optimal value for $\zeta$ is found by model optimization for each method individually.

\section{Numerical Results in the Gaussian Setting}
\label{apx:gaussresults}
In this section we report numerical comparisons of the proposed approaches in the idealized Gaussian setting~\eqref{eq:Gauss}. In that setting, since the tests and proposed estimates
only depend on the naive estimators, we can reduce each bag to its naive estimator, in other words we
can assume $N=1$ (only one observation per bag). We consider the following models for the
means $(\mu_i)_{i \in \intr{B}}$ (in each case the number of bags is $B=2000$):
\begin{itemize}
  \item Model {\bf UNIF}: ambient dimension $d=1000$, the means $(\mu_i)_{i \in \intr{B}}$ are distributed uniformly
    over the lower-dimensional cube $[-20,20]^{d'}$, $d'=10$ (the remaining coordinates are set to 0).
  \item Model {\bf CLUSTER}:  ambient dimension $d=1000$, the means are clustered in 20 clusters of centers
    $(m_i)_{i \in \intr{10}}$, drawn as $\cN(0,I_d)$, in each cluster the means are drawn as
    Gaussians $\cN(m_i,0.1 * I_d)$, 
  \item Model {\bf SPHERE}: ambient dimension $d=1000$, the 6 first coordinates of the means are distributed uniformly on the sphere of radius 50 in $\mbr^6$, the rest are set to 0.
  \item Model {\bf SPARSE}: ambient dimension $d=50$, the means are 2-sparse vectors with two random
    coordinates distributed as $\mathrm{Unif}[0,20]$.
  \end{itemize}
  In each case, we first select the parameter for the tests (parameter~$\zeta$ in~(\ref{eq:deftestsgaus}) )
  from the oracle \STBnorm\ performance. This value is held fixed and the shrinkage parameter in
  methods \MTAstb, \STBtheory, \STBweight\ is again determined as its ``oracle'' value by minimization over
  the squared error, as done in the KME experiments.

    For comparison, we also display the results of the classical positive-part James-Stein estimator (\PPJS, \citealp{baranchik1970family}),  which is a shrinkage estimator applied separately on each bag. It has no tuning parameter.
  \begin{table}[hb]
    \caption{Decrease in averaged squared estimation error compared to \NE{} in percent on the Gaussian data (higher is better).
    Averaged results over 20 trials. Standard error of one given trial is of order $5.10^{-3}$.}
    \label{tab:gaussian_results}
    \begin{center}
  % \begin{tabular}{lccccc}
  %   %\textbf{METHOD}
  %   & \MTAconst & \MTAstb & \STBnorm & \STBtheory & \STBweight\\
  %   \hline
  %   %\textbf{MODEL}    & & & & & \\
  %   % {\bf UNIF} & 0.749 & 0.790 & 0.904 & 0.908 & 0.908 \\
  %   % {\bf CLUSTER} & 0.520 & 0.985 & 0.985 & 0.985 & 0.985 \\
  %   % {\bf SPHERE} & 0.714 & 0.842 & 0.955 & 0.955 & 0.955 \\
  %   % {\bf SPARSE} & 0.153 & 0.367 & 0.398 & 0.437 & 0.439\\
  %   {\bf UNIF} & 0.427 & 0.654 & 0.797 & 0.813 & 0.813 \\
  %   {\bf CLUSTER} & 0.510 & 0.979 & 0.980 & 0.980 & 0.980 \\
  %   {\bf SPHERE} & 0.285 & 0.745 & 0.894 & 0.898 & 0.898 \\
  %   {\bf SPARSE} & 0.160 & 0.368 & 0.402 & 0.441 & 0.443 \\
  % \end{tabular}
  \begin{tabular}{lcccccc}
    %\textbf{METHOD}
    & \PPJS & \MTAconst & \MTAstb & \STBnorm & \STBtheory & \STBweight\\
    \hline
    % \textbf{MODEL}    & & & & & \\
    {\bf UNIF} & 0.439 & 0.427 & 0.653 & 0.796 & 0.813 & 0.813 \\
    {\bf CLUSTER} & 0.495 & 0.508 & 0.979 & 0.980 & 0.980 & 0.980 \\
    {\bf SPHERE} & 0.285 & 0.285 & 0.745 & 0.894 & 0.898 & 0.898 \\
{\bf SPARSE} & 0.224 & 0.162 & 0.367 & 0.402 & 0.441 & 0.443
  \end{tabular}
\end{center}
\end{table}

\vfill
\pagebreak
\bibliography{references}

\end{document}